%% file: acl_latex.tex
\documentclass[11pt]{article}

\usepackage[preprint]{acl}

\usepackage{times}
\usepackage{latexsym}

\usepackage[T1]{fontenc}

\usepackage[utf8]{inputenc}

\usepackage{microtype}

\usepackage{inconsolata}

\usepackage{graphicx}

\usepackage{booktabs}
\usepackage{multirow}
\usepackage{bm}
\usepackage{amsmath,amssymb, amsthm}
\usepackage{makecell}
\usepackage{xspace}

\usepackage{algorithm}
\usepackage{algorithmicx}
\usepackage{algpseudocode}
\usepackage{hyperref}
\usepackage{xurl}

\usepackage{subcaption}
\usepackage{setspace}
\usepackage[table]{xcolor}
\definecolor{cvprblue}{rgb}{0.21,0.49,0.74}

\usepackage[most]{tcolorbox}

\newtheorem{theorem}{Theorem}

\newtheorem{corollary}[theorem]{Corollary}
\newtheorem{proposition}[theorem]{Proposition}

\title{Harnessing Generalist Agents for Contextualized Time Series}

\author{
 \textbf{Zihao Li}, 
 \textbf{Kaifeng Jin}, 
 \textbf{Yuanchen Bei},
 \textbf{Jiaru Zou},
 \textbf{Avaneesh Kumar},
 \textbf{Xuying Ning}
 \\
 \textbf{Yanjun Zhao}, 
 \textbf{Mengting Ai},
 \textbf{Baoyu Jing}, 
 \textbf{Hanghang Tong},
 \textbf{Jingrui He}
\\
\\
 \textsuperscript{}University of Illinois Urbana-Champaign
\\
 \small{
   \textbf{Correspondence:} \href{mailto:zihaoli5@illinois.edu}{zihaoli5@illinois.edu}, \href{mailto:jingrui@illinois.edu}{jingrui@illinois.edu}
 }
}

\newcommand{\method}{\textsc{TimeClaw}\xspace}

\begin{document}
\maketitle
\begin{abstract}
Time series are often embedded in rich contexts that are essential for holistic modeling. Moreover, real-world practitioners often require end-to-end workflows for analyzing temporal dynamics, where widely studied tasks such as forecasting are only one step in a broader solution loop. While generalist AI agents offer a promising interface for such workflows under complex contexts, they still operate primarily in textual spaces that are not fully aligned with structured temporal signals. 
In this work, we introduce \method, an agentic harness framework for time series that equips generalist LLM agents with the time-series-native runtime support needed for contextualized temporal reasoning. \method integrates executable temporal tools for grounded and auditable analysis, experience-driven capability evolution for creating reusable analytical routines, and episodic multimodal memory for retrieving relevant reasoning traces. Together, these components unlock harnessed open-ended temporal reasoning with contextual information.
Extensive evaluation on multiple benchmarks covering diverse tasks across energy, finance, weather, traffic, and other real-world domains demonstrates improved performance of \method. Code is available at this \textcolor{cvprblue}{\href{https://github.com/iDEA-iSAIL-Lab-UIUC/TimeClaw}{repository hyperlink}}.
\end{abstract}

\addtocontents{toc}
{\protect\setcounter{tocdepth}{-1}}

\input{0_sections/1_introduction}

\input{0_sections/2_preliminary}

\input{0_sections/4_method}

\input{0_sections/5_experiment}

\input{0_sections/6_related_work}

\input{0_sections/7_conclusion}


\bibliography{reference}

\clearpage

\input{1_appendix/appendix}

\end{document}

%% file: 0_sections/1_introduction.tex
\section{Introduction}
Time series provide a fundamental lens for observing complex real-world system dynamics and supporting decisions over time \citep{DBLP:conf/kdd/JiangNPSNYSCNS25,wu2023timesnet,goswami2024moment}. Financial time series capture evolving market and macroeconomic conditions, guiding investment and risk-management decisions \citep{DBLP:journals/csur/ArsenaultWP25}. Healthcare time series reflect the dynamic state of the human body, enabling early detection of potential illnesses and personalized treatment planning \citep{jarrett2021clairvoyance,NIPS2016_231141b3}.
Recently, the rapid development of large language models (LLMs) and agentic AI offers new opportunities to move beyond conventional time-series modeling toward more holistic temporal understanding and reasoning~\citep{chang2025survey, li2026r}. 
This opportunity arises from two practical needs. 
First, most existing methods model time series as standalone numerical sequences, whereas real-world temporal signals are often accompanied by external contexts that are crucial for holistic modeling and can be explicitly leveraged by agents \citep{DBLP:journals/corr/abs-2502-08942, chang2026time}. 
Second, existing studies largely focus on predefined tasks, with forecasting being the most representative~\citep{goswami2024moment}, whereas real-world applications often require broader solution loops where such tasks are only intermediate steps \citep{jin2024position, DBLP:journals/corr/abs-2601-18744}.

\begin{figure}[t]
\centering
\includegraphics[width=\linewidth]{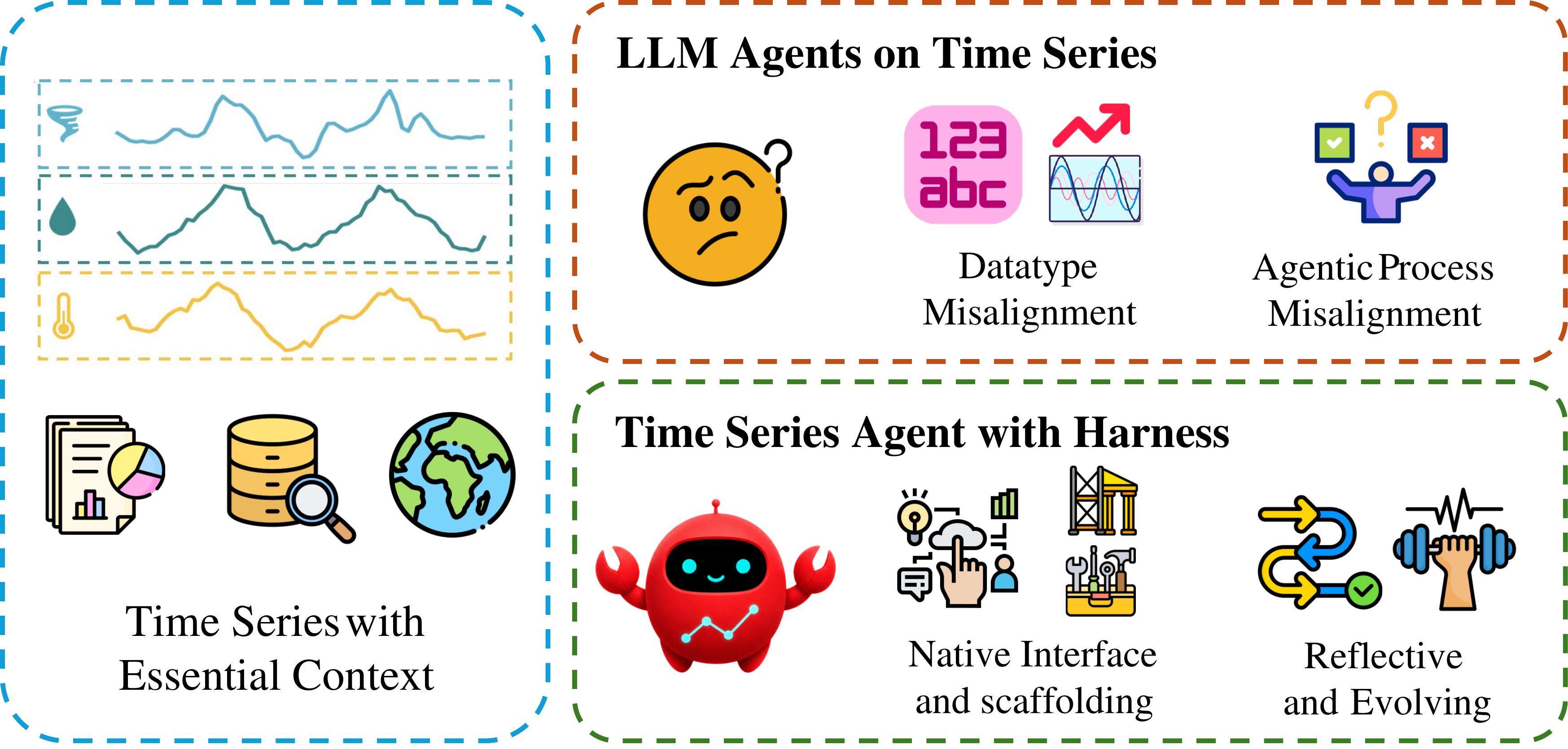}
\caption{
Contextualized time series require more than task-specific models or directly prompted LLM agents. 
\method provides a time-series-native agent harness that leverages context to support holistic modeling.
}
\label{fig:teasing}
\vspace{-5mm}
\end{figure}

However, the generality of LLMs alone does not provide the scaffold needed for reliable time-series reasoning~\citep{tan2024are,merrill-etal-2024-language}. Directly inserting time series data as ordinary text for the LLM agents remains insufficient and causes two misalignments.
The first is \textit{datatype misalignment}: temporal signals encode statistical properties such as trends and periodicity that can be distorted when serialized into tokens and processed through language-oriented attention mechanisms \citep{jin2024timellm,schwartz2024numerologic}. 
The second is \textit{agentic process misalignment}: existing agentic workflows are primarily designed for text-centric information processing \citep{DBLP:conf/iclr/YaoZYDSN023}. Long numerical sequences may further burden the context window and distract the agent from reliable temporal reasoning \citep{tan2024are}. 

To address such misalignments, agent harnesses have emerged as an effective way to equip LLM agents with reliable interfaces for perceiving and acting over non-text modalities \citep{meng2026agent, ning2026code}. Yet existing harness designs rarely account for the unique structure of time series.
This motivates our central research question:
\begin{center}
\textit{How to build time-series-native agent harness?}
\end{center}

To answer this question, we introduce \method, an agentic harness framework for contextualized time-series understanding, reasoning, and decision-making. Specifically, \method integrates executable temporal tools, experience-driven capability evolution, and episodic multimodal memory.
The executable temporal tools enable agents to inspect, transform, analyze, and act upon time series as structured temporal objects through auditable workflows. To address datatype misalignment, \method operates on time-series-native runtime instead of serialized text. The experience-driven capability evolution mechanism allows the agent to expand its temporal reasoning abilities over time and address agentic process misalignment. The episodic multimodal memory module retrieves relevant past reasoning traces to guide new tasks with similar temporal patterns or contexts. 
Extensive experiments across diverse benchmarks and real-world domains show that \method consistently improves end-to-end performance while maintaining token efficiency.
Our contributions are summarized as follows:

\begin{itemize}

\item We identify the limitations of directly using LLMs for time-series reasoning, showing that textualizing temporal signals induces data-type and agentic-process misalignments.

\item We introduce \method, a time-series-native agentic harness that integrates executable tools, experience-driven capability evolution, and episodic multimodal memory.

\item We conduct extensive experiments across diverse real-world domains, demonstrating the effectiveness of \method for end-to-end reasoning over contextualized time series.

\end{itemize}

%% file: 0_sections/2_preliminary.tex
\section{Background}
\label{sec:preliminary}

\begin{figure*}[t]
\centering
\includegraphics[width=\linewidth]{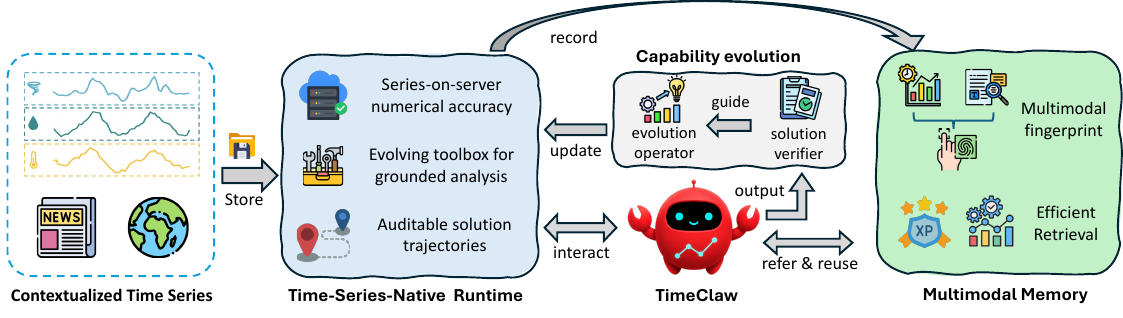}
\caption{
Overview of \method. 
Given contextualized time series, \method operates through a time-series-native runtime that provides server-side numerical execution, an evolving toolbox for grounded analysis, and auditable solution trajectories (Section \ref{sec:tools}).
The agent further improves over time through capability evolution (Section \ref{sec:evolve}) and retrieves relevant experience from multimodal memory via multimodal fingerprints (Section \ref{sec:memory}).
}
\label{fig:main}
\vspace{-3mm}
\end{figure*}

\subsection{Preliminary}

We use calligraphic letters (e.g., $\mathcal{A}$) for sets and bold capital letters for matrices (e.g., $\bm{A}$). For matrix indices, $\bm{A}[i, j]$ denotes the entry in the $i^{\textrm{th}}$ row and the $j^{\textrm{th}}$ column. For a vector $\bm{v}$, $v[i:j]$ represents the sub-vector sliced from the $i^{\textrm{th}}$ to the $j^{\textrm{th}}$ position, inclusively. $\bm{A}[i, :]$ returns the $i^{th}$ row in $\bm{A}$ and $\bm{A}[:i]$ returns the first $i$ rows of $\bm{A}$.

\paragraph{Time Series and Contextualized Time Series.} A \textit{numerical time series} is denoted as
\begin{equation}
    \mathbf{X}_{1:T} = [\mathbf{x}_1, \mathbf{x}_2, \ldots, \mathbf{x}_T] \in \mathbb{R}^{N \times T},
\end{equation}
where $T$ is the number of timestamps, $N$ is the number of variables, and $\mathbf{x}_t \in \mathbb{R}^{N}$ is the observation at timestamp $t$. For $1 \leq a \leq b \leq T$, we denote the temporal slice from timestamp $a$ to $b$ as
\begin{equation}
    \mathbf{X}_{a:b}
    =
    [\mathbf{x}_a, \mathbf{x}_{a+1}, \ldots, \mathbf{x}_b].
\end{equation}
We define a \textit{contextualized time series} as
\begin{equation}
    \mathcal{X}
    =
    (\mathbf{X}_{1:T}, \mathcal{C}),
\end{equation}
where $\mathbf{X}_{1:T}$ is the numerical temporal signal and $\mathcal{C}$ is the associated context. The context $\mathcal{C}$ may contain information from various sources, such as textual descriptions, categorical metadata, external constraints, and domain knowledge. In this work, we focus on the text-time series mixture, a special case of contextualized time series. This setting is broadly representative of many real-world temporal reasoning scenarios where textual cues provide essential information for interpreting temporal dynamics and supporting decisions. Yet the core methodology is generalizable across additional context sources beyond natural language.
\paragraph{From Predefined Tasks to Open-Ended Temporal Reasoning.} Conventional time-series learning is usually defined by fixed task mappings, such as $\mathbf{X}_{1:T} \mapsto \mathbf{X}_{T+1:T+H}$ for forecasting or $\mathbf{X}_{1:T} \mapsto y$ for classification. We define \textit{open-ended contextualized temporal reasoning} as a more general setting where the task objective and output format are specified by a natural-language instruction, and the numerical time series is interpreted together with its associated context. Given a contextualized time series
$\mathcal{X}=(\mathbf{X}_{1:T},\mathcal{C})$, we define a task instance as
\begin{equation}
    \tau = (q, \mathcal{X}, \mathcal{Y}, y^\star, \ell),
\end{equation}
where $q$ is the instruction, $\mathcal{Y}$ is the output space, $y^\star \in
\mathcal{Y}$ is the target, and $\ell:\mathcal{Y}\times\mathcal{Y}\to\mathbb{R}$
is the evaluation function. The goal is to produce
\begin{equation}
    \hat{y}=f(q,\mathcal{X})\in\arg\min_{y\in\mathcal{Y}}\ell(y,y^\star).
\end{equation}
Here $\mathcal{Y}$ may include numerical predictions, labels, rankings, textual
explanations, decisions, executable actions, or their combinations.

\subsection{Why NOT Fully Trust LLMs?}
\textbf{LLM Workflow for Time Series.}
Existing LLM-based agents usually process contextualized time series through
textual serialization. Given $\mathcal{X}=(\mathbf{X}_{1:T},\mathcal{C})$, a
serialization function $\sigma$ maps the numerical series and context to text
$s_\mathcal{X}=\sigma(\mathcal{X})$. The LLM then tokenizes the instruction and serialized input as
\begin{equation}
    \mathbf{z}_{1:L}=\operatorname{Tokenize}([q;s_{\mathcal{X}}]),
\end{equation}
\textbf{Limitations.}
This text-centric workflow introduces two forms of misalignment. First,
\textit{data-type misalignment} arises because numerical temporal signals are
converted into discrete text tokens. As a result, numerical proximity, temporal
resolution, long-range dependencies, and fine-grained variations may not be
faithfully preserved under tokenization and finite context budgets. Second,
\textit{agentic-process misalignment} arises because the agent reasons over
serialized text rather than native temporal objects. Operations that are natural
for time series, such as slicing, aggregation, decomposition, smoothing,
forecasting, anomaly localization, and numerical verification, become indirect
language-based reasoning steps. Due to page limitations, we provide a more detailed analysis of these
limitations in Appendix~\ref{ap:analysis}, including numerical-token
distance mismatch, decimal place-value distortion, next-token versus next-value
objective mismatch, and long-context temporal information dilution. 

%% file: 0_sections/4_method.tex
\section{Time-Series-Native Agent Harness}
\label{sec:method}

We address the misalignment between language agents and time-series tasks through an \textit{agent harness}~\cite{ning2026code}: a time-series-native runtime that surrounds a frozen LLM policy and reshapes how it interacts with contextualized temporal data. Instead of forcing numerical signals into the token stream, the harness hosts temporal objects in a structured runtime environment and exposes them only through typed executable interfaces. The agent analyzes them through auditable temporal operations, while intermediate analytical states are maintained in a structured workspace.

\subsection{Harness Framework Overview}
\label{sec:overview}

\method instantiates this harness through three coordinated components. \textit{Executable temporal tools} (Section~\ref{sec:tools}) provide typed interfaces for accessing and analyzing time-series objects, producing grounded and auditable analytical evidence. \textit{Experience-driven capability evolution} (Section~\ref{sec:evolve}) abstracts recurring procedures as useful temporal routines to be reused rather than repeatedly rediscovered. \emph{Episodic multimodal memory} (Section~\ref{sec:memory}) stores past reasoning traces and retrieves relevant experience through complementary textual and temporal modalities.
Formally, \method augments a frozen LLM policy $\pi_\theta$ with a harness $\mathcal{H}$:
\begin{equation}
    \mathcal{H} = (\mathcal{T}, \mathcal{E}, \mathcal{M} ),
\end{equation}
where $\mathcal{T}$ is the set of executable temporal tools, $\mathcal{E}$ is the capability-evolution loop and $\mathcal{M}$ is the episodic multimodal memory.

\paragraph{Workspace and rollout.}
For each task $\tau=(q,\mathcal{X},\mathcal{Y},y^\star,\ell)$, $\mathbf{X}_{1:T}$ is loaded into a task-local \emph{workspace} $\mathcal{W}$ hosted by the harness runtime that
can be accessed through protocols (e.g., Model Context Protocol \citep{mcp2024}). This
workspace preserves the series at full numerical precision. The policy interacts with the harness through a
trajectory
\begin{equation}
    \zeta=(a_1,o_1,\ldots,a_K,o_K,\hat{y}),
\end{equation}
where each action $a_r\in\mathcal{A}$ denotes an agent decision, such as executing a runtime operation or performing a reasoning step. The corresponding observation $o_r\in\mathcal{O}$ records the outcome of this decision, either from the harness or the agent's internal reasoning. The policy selects
\begin{equation}
    a_r\sim\pi_\theta(\cdot\mid q,\mathcal{C},o_{1:r-1}),
\end{equation}
conditioned on the instruction, context, and accumulated evidence, ultimately emitting $\hat{y}\in\mathcal{Y}$.

\subsection{Runtime-Native Temporal Tools}
\label{sec:tools}

Despite differing in domains and objectives, time-series tasks often rely on a shared list of common temporal operations. Unlike conventional tool-augmented agents that must pass data through language messages, \method executes these operations inside the harness runtime, where the full time-series object already resides in the workspace $\mathcal{W}$. Thus, the agent only specifies the intended operation and its parameters without the need to spend context budget on tool interaction. Each tool runs directly on $\mathcal{W}$ at numerical precision and returns a compact structured observation. We initialize $\mathcal{T}$ with common time series tools and allow new tools to emerge through capability evolution, as summarized in Table~\ref{tab:tools}.

\paragraph{Auditable trajectories as reusable experience.}
The harness requires every numerical claim to be grounded in a returned tool observation rather than inferred from free-form language reasoning.
Each tool-related pair $(a_r,o_r)$ in the rollout also records the analytical intent behind that invocation. The resulting evidence chain ties the final answer $\hat{y}$ to verifiable temporal operations and forms a reusable problem-solving trace, which further supports capability evolution and episodic memory.

\subsection{Experience-Driven Capability Evolution}
\label{sec:evolve}

Runtime-native tools cover common temporal primitives, but real workflows often reveal recurring task-specific routines that cannot all be anticipated at design time.
The second facet of the harness turns such repeated experience into executable capability. Specifically, an evolution operator $\mathcal{E}$ identifies verified recurring sub-procedures and admits them into the tool set $\mathcal{T}$. As a result, $\mathcal{T}$ evolves and expands with the target task distribution.

\paragraph{Capability evolution mechanism.}
We implement the evolution operator $\mathcal{E}$ as a code-specialized LLM. Given clustered trajectories and task descriptions, $\mathcal{E}$ abstracts recurring analytical sub-procedures into self-contained executable tools, each with an invocation docstring and an implementation that operates on the runtime workspace $\mathcal{W}$. Each candidate tool is verified by execution on a held-out subset of similar tasks and is admitted into $\mathcal{T}$ only if its held-out success rate exceeds a threshold $\gamma$. Once admitted, the tool is immediately registered alongside the seed tools and can be invoked in future rollouts.

\subsection{Episodic Multimodal Memory}
\label{sec:memory}

\input{tables/cik_main}

Episodic memory is common in agentic systems, but its design is nontrivial in contextualized time-series settings, since relevance depends on both task framing and temporal structure. A text-based or jointly serialized retrieval space can (1) retrieve misleading precedents, since token similarity does not necessarily preserve numerical proximity (Proposition~\ref{prop:token_numeric}); (2) risk modality imbalance for long series; (3) exhibit scale instability, as normalization choices can change which temporal patterns appear similar.

\paragraph{Memory records.}
Each memory record $m\in\mathcal{M}$ stores a past task, its retrieval keys, and its successful analytical trace:
\begin{equation}
    m = (\tau, \phi, \psi, \zeta, \rho),
\end{equation}
where $\tau$ is the original task, $\zeta$ is the rollout trajectory produced under the harness, and $\rho$ is a short transferable rationale summarizing the contextual cue and analytical rule behind the solution. For example, in the energy domain, $\rho$ may note that a snowstorm scenario implies a sharp electricity demand increase due to air-conditioning load. The remaining fields $\phi$ and $\psi$ serve as retrieval keys over the signal and text modalities, respectively.

\paragraph{Multimodal retrieval keys.}
For holistic memory retrieval, we design retrieval keys for both modalities. The context key $\phi=\Phi(c_\tau)$ embeds a textual descriptor $c_\tau$ constructed from the task context. The time series key $\psi=\Psi(\mathbf{X}_{1:T})$ is a time-series fingerprint that encodes structural, statistical, spectral, temporal-dynamic, and multivariate properties. 
Due to page limitations, we provide details of fingerprint computation in Appendix \ref{ap:fingerprint}.

\paragraph{Similarity computation and retrieval.}
Given a query task $\tau_q$ with keys $(\psi_q,\phi_q)$, the text-side similarity is defined by cosine similarity,
\begin{equation}
    s_{\mathrm{text}}(q,m)=\cos(\phi_q,\phi_m),
\end{equation}
while the time series similarity is defined by the negative distance between normalized fingerprints,
\begin{equation}
    s_{\mathrm{ts}}(q,m)=-\|\tilde{\psi}_q-\tilde{\psi}_m\|_2,
\end{equation}
where $\tilde{\psi}$ denotes the normalized fingerprint. These two scores can be merged in different ways, such as weighted aggregation, reciprocal-rank fusion, or staged retrieval. In \method, we adopt a simple yet effective two-stage top-k retrieval strategy with details provided in Appendix \ref{ap:multimodal_memory}.

%% file: tables/cik_main.tex
\begin{table*}[t]
\caption{Overall results on the CiK benchmark. We report token usage when applicable. The best and second-best results in each metric column are highlighted in \textbf{bold} and \underline{underlined}, respectively. \method achieves the best average RCRPS and sMAPE, obtains the best RCRPS in 4 out of 5 context types, and is nearly tied regarding the remaining context type. Meanwhile, \method achieves substantially better performance with nearly half the token budget compared with the second-best multi-agent solution.}
\label{table:main-results}
\centering
\resizebox{\textwidth}{!}{%
\begin{tabular}{lcccccccc}
\toprule
 \multirow{2}{*}{\makecell{Method}} 
 & \multirow{2}{*}{\makecell{Average \\ RCRPS ($\downarrow$)}} 
 & \multirow{2}{*}{\makecell{Average \\ sMAPE ($\downarrow$)}} 
 & \multirow{2}{*}{\makecell{Average \\ Token Usage ($\downarrow$)}}
 & \multicolumn{5}{c}{RCRPS by Context Type ($\downarrow$)} \\
\cmidrule(lr){5-9}
 & & & & Intemporal & Historical & Future & Covariate & Causal \\
\midrule
\multicolumn{9}{l}{\textit{Traditional Time Series Models}} \\
~~~ARIMA \citeyearpar{broomhead1989time} & 0.2772 & 91.2514 & -- & 0.2987 & 0.1032 & 0.1577 & 0.2073 & 0.2822 \\
~~~ETS \citeyearpar{jain2017study} & 0.3282 & 84.5230 & -- & 0.3654 & 0.1738 & 0.1966 & 0.2402 & 0.3290 \\
~~~DLinear \citeyearpar{DBLP:conf/aaai/ZengCZ023} & 2.3448 & 184.1755 & -- & 3.0988 & 4.9004 & 2.6461 & 1.9710 & 3.0394 \\
~~~PatchTST \citeyearpar{DBLP:conf/iclr/NieNSK23} & 2.3374 & 166.8354 & -- & 3.0980 & 4.9007 & 2.6294 & 1.9614 & 3.0523 \\
\midrule
\multicolumn{9}{l}{\textit{Time Series Foundation Models}} \\
~~~Chronos-2 \citeyearpar{DBLP:journals/corr/abs-2510-15821} & 0.2344 & 90.8453 & -- & 0.2277 & 0.1852 & 0.2262 & 0.2406 & 0.3426 \\
~~~Lag-Llama \citeyearpar{rasul2023lag} & 0.2628 & 95.1976 & -- & 0.2475 & 0.1233 & 0.1987 & 0.2461 & 0.3299 \\
~~~Moirai-Large \citeyear{DBLP:conf/icml/WooLKXSS24} & 0.4545 & 107.8152 & -- & 0.3865 & 0.0929 & 0.3936 & 0.3792 & 0.2690 \\
\midrule
\multicolumn{9}{l}{\textit{LLM/Agent on Time Series}} \\
~~~UniTime \citeyearpar{DBLP:conf/www/LiuHLDLHZ24} & 0.2822 & 93.3854 & -- & 0.2910 & \underline{0.0889} & 0.1665 & 0.2964 & 0.2677 \\
~~~Time-LLM \citeyearpar{DBLP:conf/iclr/0005WMCZSCLLPW24} & 0.3990 & 102.6825 & -- & 0.3629 & 0.1144 & 0.3172 & 0.3589 & 0.3146 \\
~~~TS-Agent \citeyearpar{DBLP:journals/corr/abs-2510-07432} & 0.1421 & 61.2980 & 47455.5880 & 0.1434 & 0.0904 & 0.1210 & 0.1516 & \textbf{0.1596} \\
~~~TSci \citeyearpar{DBLP:journals/corr/abs-2510-01538} & 0.1448 & 69.9068 & 47905.1056 & 0.1540 & 0.1001 & 0.1181 & 0.1694 & 0.1971 \\
\midrule
\multicolumn{9}{l}{\textit{General Agentic Pipelines}} \\
~~~Direct Prompt & 0.1703 & 60.4407 & 24861.3451 & 0.1938 & 0.1244 & 0.1239 & 0.1880 & 0.2793 \\
~~~CoT Prompt \citeyearpar{DBLP:conf/nips/Wei0SBIXCLZ22} & 0.1726 & 66.3724 & 27120.8662 & 0.1771 & 0.1230 & 0.1520 & 0.1932 & 0.2392 \\
~~~ReAct \citeyearpar{DBLP:conf/iclr/YaoZYDSN023} & 0.1514 & 66.2284 & 25502.9507 & 0.1678 & 0.1452 & 0.1317 & 0.1681 & 0.2249 \\
~~~Self-Reflection \citeyearpar{DBLP:journals/corr/abs-2405-06682} & 0.1608 & 64.7013 & 28942.8662 & 0.1869 & 0.1240 & 0.1131 & 0.1851 & 0.2756 \\
~~~Multi-Agent Reflection \citeyearpar{DBLP:conf/icml/YuanX25} & \underline{0.1294} & \underline{55.5704} & 63033.0070 & \underline{0.1346} & 0.1086 & \underline{0.1108} & \underline{0.1437} & 0.1697 \\
\midrule
\method (Ours) & \textbf{0.1145} & \textbf{52.4542} & 35553.2113 & \textbf{0.1303} & \textbf{0.0886} & \textbf{0.0703} & \textbf{0.1239} & \underline{0.1597} \\
\bottomrule
\end{tabular}
}
\vspace*{-0.2cm}
\end{table*}

%% file: 0_sections/5_experiment.tex
\section{Experiments}

\input{tables/tsr_main.tex}

We comprehensively evaluate \method on three recent benchmarks spanning diverse application domains, demonstrating its effectiveness across a wide range of open-ended contextualized time-series scenarios.

\noindent\textbf{Benchmarks and Evaluation Focus.}
We evaluate \method on three complementary benchmarks. Context-is-Key (CiK) \citep{DBLP:conf/icml/WilliamsAMZSRRL25} examines context utilization in context-aware time-series reasoning; TSRBench \citep{DBLP:journals/corr/abs-2601-18744} evaluates open-ended generalist capability through multi-task multimodal scenarios; and TSAIA \citep{DBLP:journals/corr/abs-2509-01822} focuses on practical domain-specific financial analysis. We provide details in Appendix \ref{ap:dataset_details}

\noindent\textbf{Baselines.}
We compare \method with a broad set of baselines spanning five families. 
(a) \textit{Traditional time-series models}, including statistical forecasting methods ARIMA and ETS \citep{shchur2023autogluon}, as well as neural forecasting models DLinear \citep{DBLP:conf/aaai/ZengCZ023} and PatchTST \citep{DBLP:conf/iclr/NieNSK23}. 
(b) \textit{Time-series foundation models}, including Chronos-family models \citep{DBLP:journals/tmlr/AnsariSTZMSSRPK24, DBLP:journals/corr/abs-2510-15821}, Lag-Llama \citep{rasul2023lag}, and Moirai-family models \citep{DBLP:conf/icml/WooLKXSS24}. 
(c) \textit{LLMs and agents designed for time series}, including UniTime \citep{DBLP:conf/www/LiuHLDLHZ24}, Time-LLM \citep{DBLP:conf/iclr/0005WMCZSCLLPW24}, TS-Agent \citep{DBLP:journals/corr/abs-2510-07432}, and TSci \citep{DBLP:journals/corr/abs-2510-01538}.
(d) \textit{Open-source LLMs}, such as LLaMA3-70B \citep{DBLP:journals/corr/abs-2407-21783}.
(e) \textit{General agentic pipelines}, including direct prompting, chain-of-thought prompting \citep{DBLP:conf/nips/Wei0SBIXCLZ22}, ReAct \citep{DBLP:conf/iclr/YaoZYDSN023}, self-reflection \citep{DBLP:journals/corr/abs-2405-06682}, and multi-agent reflection \citep{DBLP:conf/icml/YuanX25}. 
Additional details are provided in Appendix~\ref{ap:baselines}.

\begin{figure}[t]
\centering
\includegraphics[width=\linewidth]{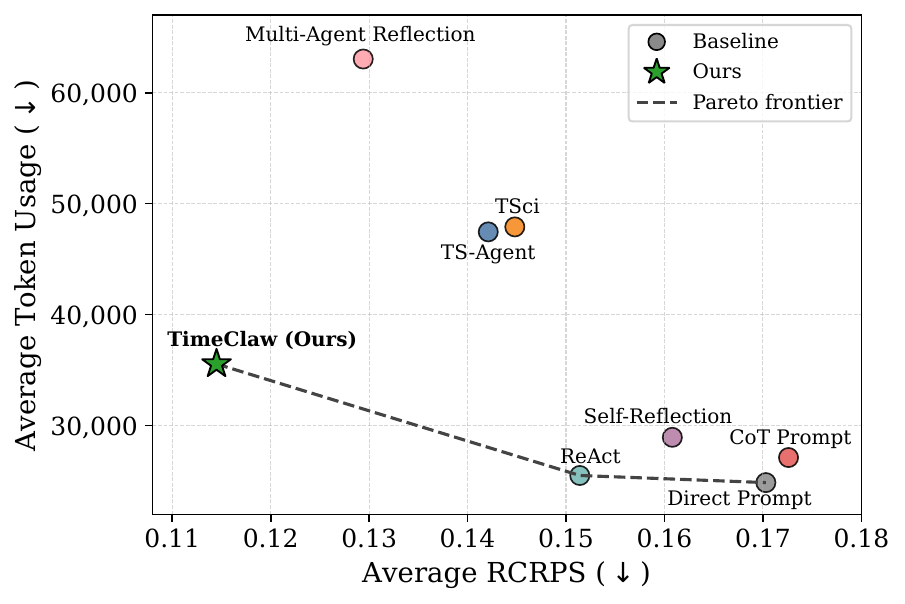}
\caption{Performance-efficiency trade-off on CiK. Among methods with reported token usage, \method achieves the best RCRPS with substantially lower token consumption than the strongest multi-agent baseline.}
\label{fig: teasing}
\end{figure}

\subsection{Context Matters for Time Series, and \method Uses It More Effectively}

\noindent\textbf{Metrics.}
RCRPS is the primary metric of CiK. It emphasizes context-relevant regions and penalizes violations of task-specific constraints. We additionally report sMAPE as a widely used metric in time series. We also report token usage when applicable. Detailed metric definitions are provided in Appendix~\ref{ap:metrics}.

\noindent\textbf{Main Results.}
Table~\ref{table:main-results} summarizes the results on the CiK benchmark. Overall, methods that can access and reason over contextual information generally outperform context-agnostic models. DLinear and PatchTST show large errors because they require the training length to exceed the prediction length, which can fail in open-ended tasks.
\method achieves the best average RCRPS and sMAPE, improving average RCRPS by 11.5\% relative to the strongest baseline. \method obtains the best RCRPS in 4 out of 5 context categories and is nearly tied with the best baseline in the remaining one. Notably, \method achieves these gains efficiently. Compared with Multi-Agent Reflection, \method uses 43.6\% fewer tokens while achieving better average RCRPS and sMAPE. We provide additional domain-wise results in Table~\ref{tab:cik_resuls_domain}.

\subsection{\method Improves Open-Ended Temporal Reasoning}

We next evaluate \method on TSRBench~\citep{DBLP:journals/corr/abs-2601-18744}, a generalist multi-task multimodal benchmark spanning perception, reasoning, prediction, and decision making. Numerical-only time-series models cannot handle end-to-end scenarios.

Table~\ref{tab:tsrbench_nano_main} compares \method with representative agentic frameworks, all using GPT-5-nano as the base model. \method achieves the best average accuracy with a 15.8\% relative improvement while using the fewest tokens.
Across task categories, \method achieves the best performance except for ranking second on perception. This suggests that although serialization can encode sufficient information under a large context window, agents may still struggle to decode it into actionable insights. With agent harness, \method yields strong gains on reasoning and prediction, where time-series-specific analysis is more critical.

\input{tables/tsr_main_open.tex}

We further compare \method with various open-source models, including general LLMs, multimodal vision-language models (VLMs), and time-series LLMs (TSLLMs). Time series are represented as textual numerical sequences for LLMs, plots for VLMs, and projector-based embeddings for TSLLMs.
As shown in Table~\ref{tab:tsrbench_nano_open}, \method achieves the best performance across all task categories, outperforming larger open-source models with up to 235B parameters with GPT-5-nano.

\subsection{\method Applicable in Practice}

\begin{figure}[t]
\centering
\includegraphics[width=\linewidth]{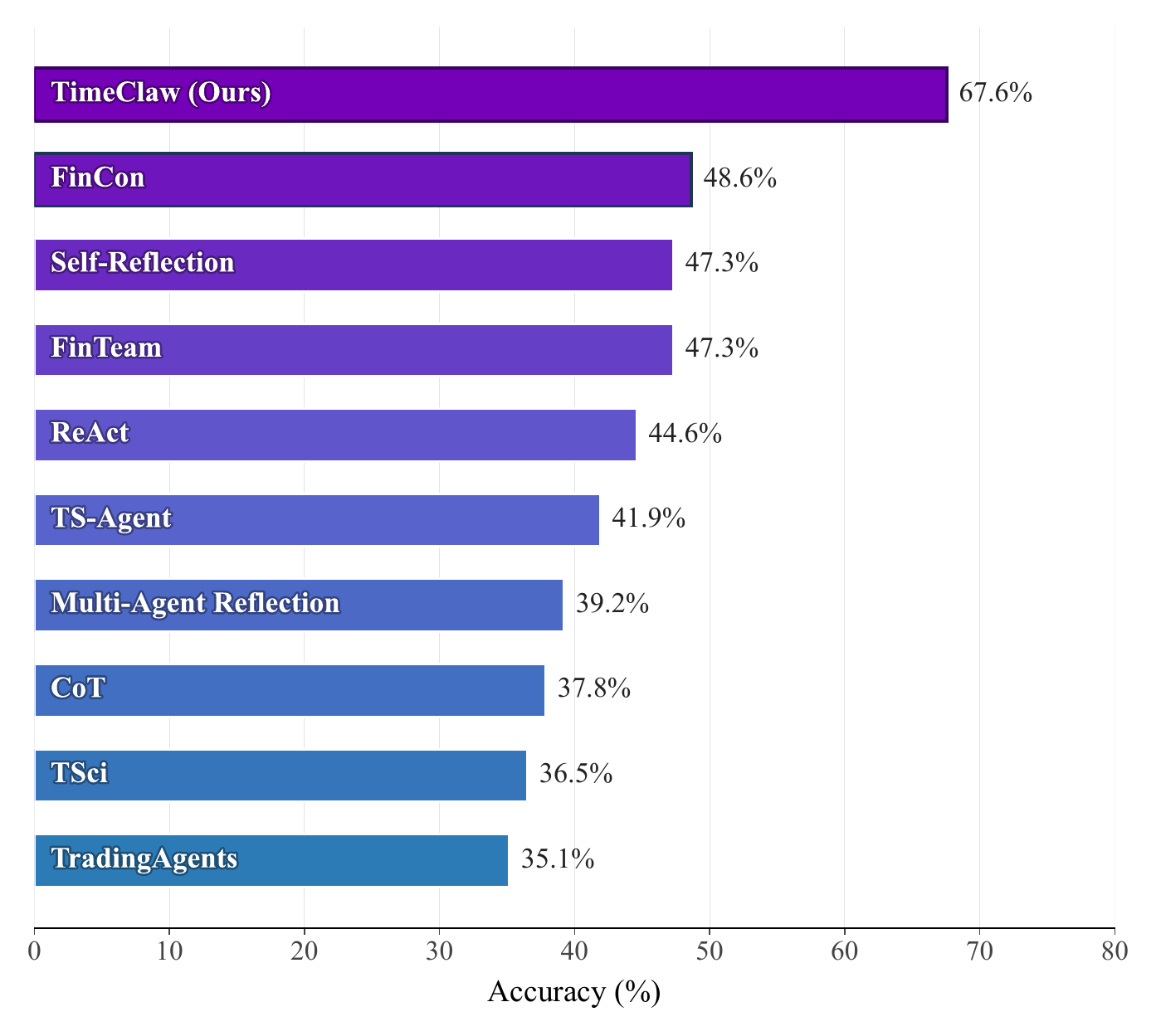}
\caption{
Performance on TSAIA. 
\method substantially outperforms both general agentic baselines and finance-specific agents, demonstrating its practical applicability in financial time-series analysis.
}
\label{fig:tsaia_mc_accuracy}
\end{figure}

We finally evaluate whether \method remains effective in a practical domain-specific setting. 
We use TSAIA~\citep{DBLP:journals/corr/abs-2509-01822}, a financial time-series analysis benchmark.
In addition to general agentic baselines, we include three finance-specific agent systems: FinCon~\citep{DBLP:conf/nips/YuYLDJCCSCLXZSX24}, FinTeam~\citep{DBLP:conf/nlpcc/WuWLYLZLCZW25}, and TradingAgents~\citep{DBLP:journals/corr/abs-2412-20138}. 

\noindent\textbf{Main Results.}
Figure~\ref{fig:tsaia_mc_accuracy} reports the multiple-choice accuracy on TSAIA. Notably, through the capability-evolution loop in Section~\ref{sec:evolve}, \method autonomously evolves three reusable routines from bank-building trajectories: \texttt{portfolio\_var}, \texttt{portfolio\_sharpe}, and \texttt{capm\_regression}. 
Equipped with these evolved routines, \method substantially outperforms all baselines by 38.9\% relative improvement. demonstrating that \method can acquire practical domain capabilities from experience.

\subsection{Ablation Study}

We conduct ablation studies to examine the effects of memory retrieval, backbone models, and framework components. 
As shown in Figure~\ref{fig:ablation_study}(a), enabling episodic memory retrieval improves both sMAPE and RCRPS, and larger retrieval sizes further improve performance. This shows that relevant past reasoning trajectories provide useful guidance for new contextualized time-series tasks.
Figure~\ref{fig:ablation_study}(b) shows that \method benefits from stronger LLM backbones, suggesting that time-series-native support complements general reasoning ability. 
Finally, Figure~\ref{fig:ablation_study}(c) shows that removing the tool harness, capability evolution, or memory all degrades performance, confirming that each part contributes to the final result. Due to page limitations, we provide further ablation studies in Appendix \ref{ap:full_experiments}, including memory bank size ablation and model family compatibility ablation.

\subsection{Case Study}

We provide detailed case studies in Appendix~\ref{ap:case_study} to illustrate how \method works on concrete tasks. 
These examples show how \method leverages textual context under identical numerical histories, uses retrieved memory to guide causal reasoning, and evolves reusable tools from prior trajectories. 
They qualitatively demonstrate that the gains of \method come from interpretable uses of context, memory, and experience-driven capability evolution.

\begin{figure*}[t]
  \begin{subfigure}[t]{0.32\linewidth}
    \centering
    \includegraphics[width=\linewidth]{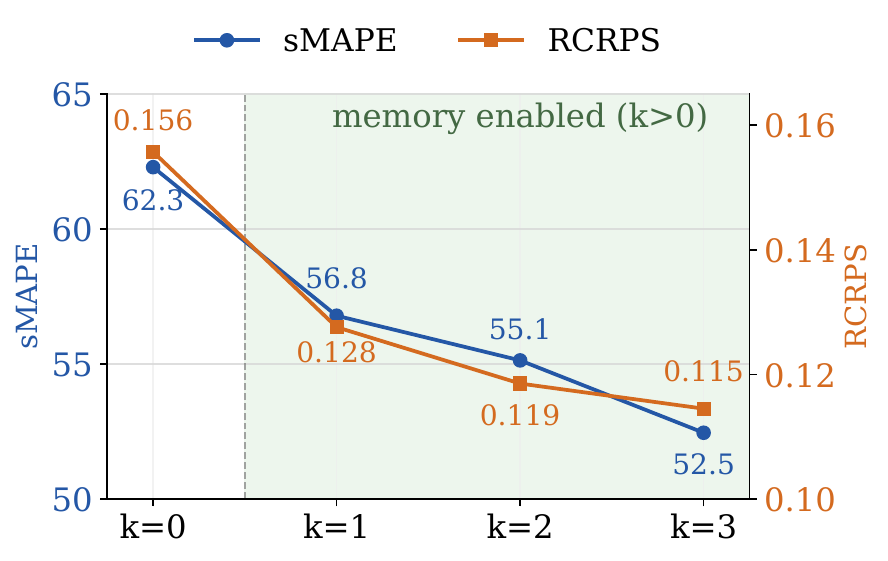}
    \caption{Retrieval Size Ablation.}
  \end{subfigure}\hfill
  \begin{subfigure}[t]{0.32\linewidth}
    \centering
    \includegraphics[width=\linewidth]{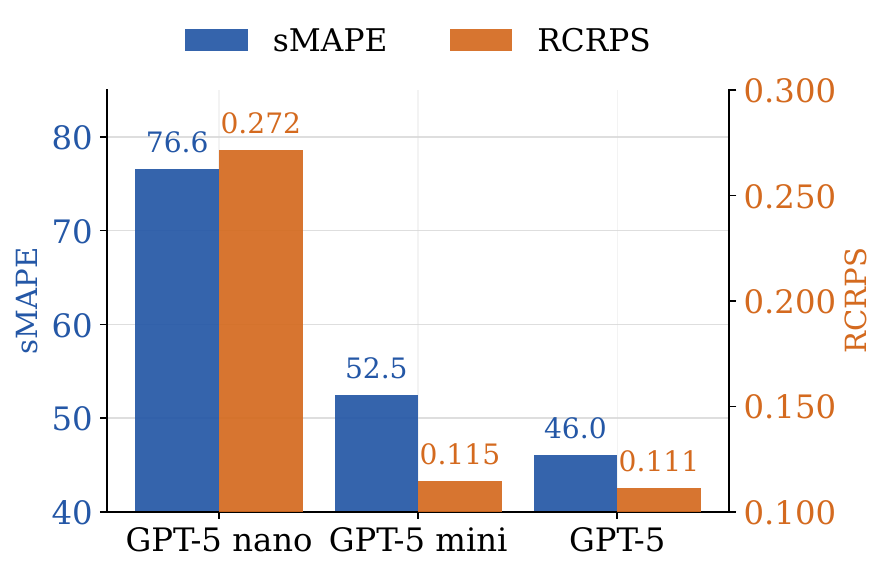}
    \caption{Backbone Model Ablation.}
  \end{subfigure}\hfill
  \begin{subfigure}[t]{0.32\linewidth}
    \centering
    \includegraphics[width=\linewidth]{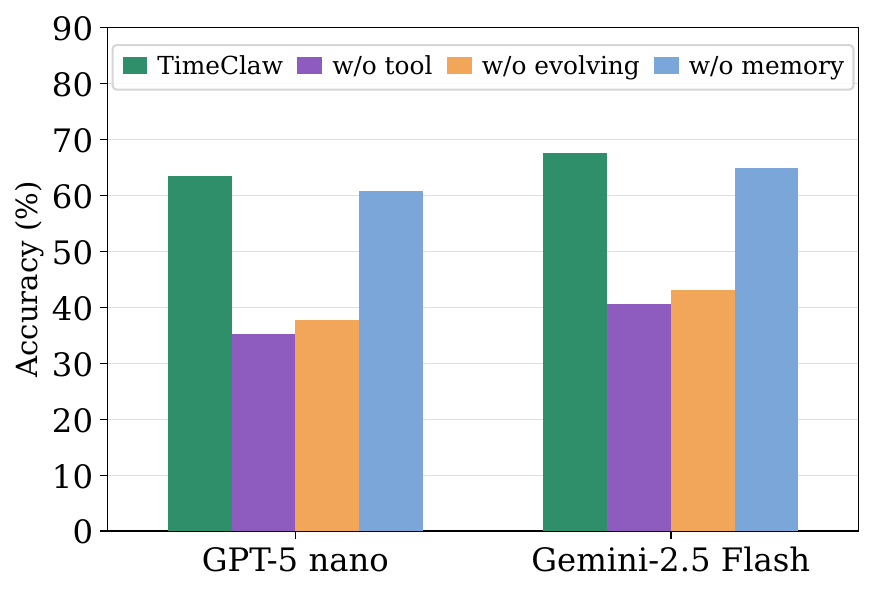}
    \caption{Component Ablation.}
  \end{subfigure}\\[4pt]

\caption{
Ablation studies of \method. 
\textbf{(a)} Retrieval size ablation shows the benefit of enabling memory retrieval. 
\textbf{(b)} Backbone model ablation shows that \method consistently benefits from stronger LLM backbones. 
\textbf{(c)} Component ablation verifies the contributions of the tool harness, capability evolution, and memory. 
}
\label{fig:ablation_study}
\end{figure*}

%% file: tables/tsr_main.tex
\renewcommand{\arraystretch}{1.2}
\begin{table*}[t]
\caption{Performance on TSRBench across open-ended temporal reasoning tasks. The best and second-best results in each metric column are highlighted in \textbf{bold} and \underline{underlined}, respectively. All methods use GPT-5-nano as the base model. \method achieves the best average accuracy while using the fewest tokens.}
\label{tab:tsrbench_nano_main}
\centering
\resizebox{\textwidth}{!}{%
\begin{tabular}{lcccccc}
\toprule
Method & Perception & Reasoning & Prediction & Decision Making & Avg. Accuracy (\%) & Avg. Token Consumption \\ 
\midrule
CoT \cite{wei2022chain} & 48.57 & 36.84 & 46.30 & 28.57 & 40.05 & 22,274 \\
ReAct \cite{DBLP:conf/iclr/YaoZYDSN023} & 55.71 & 36.26 & 43.52 & \underline{31.75} & 40.78 & 25,165 \\
Self-Reflection \cite{DBLP:journals/corr/abs-2405-06682} & 50.00 & \underline{40.93} & 38.89 & 30.16 & 40.29 & 25,799 \\
Multi-Agent Reflection \cite{DBLP:conf/icml/YuanX25} & 62.85 & 33.92 & \underline{54.63} & 25.40 & \underline{42.96} & 45,279 \\
TSci \cite{DBLP:journals/corr/abs-2510-01538} & \textbf{65.71} & 33.33 & 43.52 & 26.98 & 40.53 & 51,273 \\
\midrule
\method (Ours)    & \underline{64.29} & \textbf{43.86} & \textbf{59.26} & \textbf{33.33} & \textbf{49.76} & \textbf{15,979} \\
\bottomrule
\end{tabular}%
}
\end{table*}
\renewcommand{\arraystretch}{1.0}

%% file: tables/tsr_main_open.tex
\begin{table}[t]
\caption{\method compared with open-source models on TSRBench. Baseline results are taken from the original TSRBench paper~\citep{DBLP:journals/corr/abs-2601-18744}. By harnessing GPT-5-nano, \method achieves the best performance, outperforming large open-source models with up to 235B parameters. Full results in Table \ref{tab:tsrbench_ap_open}.}
\label{tab:tsrbench_nano_open}
\centering
\resizebox{1.0\columnwidth}{!}{%
\begin{tabular}{lccccc}
\toprule
Method & Percept & Reason & Predict & Decide & Overall \\ 
\midrule
\textit{General LLMs} & & & & & \\
Qwen2.5-3B                   & 45.4 & 27.3 & 40.4 & 23.3 & 33.2 \\
Qwen2.5-7B                   & 49.6 & 28.3 & 34.8 & 28.5 & 33.7 \\
Qwen2.5-72B                  & 55.6 & 39.7 & 43.9 & 32.2 & 42.4 \\
Qwen3-1.7B                   & 47.0 & 27.7 & 48.7 & 28.2 & 36.6 \\
Qwen3-8B                     & 52.6 & 29.1 & 47.4 & 32.0 & 38.3 \\
Qwen3-235B-A22B              & 63.8 & 40.9 & 35.6 & 32.5 & 42.1 \\
Gemma3-12B-it                & 54.3 & 30.6 & 37.7 & 32.9 & 36.8 \\
Gemma3-27B-it                & 57.0 & 33.7 & 38.3 & 31.8 & 38.6 \\
\midrule
\textit{Multimodal LLMs} &  &  &  &  & \\
Phi4-Multimodal-8B           & 49.0 & 26.0 & 30.0 & 32.1 & 31.9 \\
InternVL3.5-1B               & 45.9 & 27.8 & 38.6 & 28.2 & 33.8 \\
InternVL3.5-8B               & 59.3 & 35.6 & 39.4 & 28.2 & 39.5 \\
InternVL3.5-38B              & 60.3 & 38.1 & 32.3 & 32.3 & 39.5 \\
MiniCPM-V-4.5-8B             & 61.3 & 26.5 & 39.0 & 27.9 & 35.9 \\
MiMo-VL-7B-RL                & 59.1 & 33.1 & 33.2 & 30.9 & 37.2 \\
\midrule
\textit{Time Series LLMs} & & & & & \\
OpenTSLM-tsqa-sp-3B          & 39.7 & 26.8 & 33.5 & 28.5 & 31.0 \\
ChatTS-14B                   & 50.7 & 29.3 & 30.2 & 28.5 & 33.0 \\
TS-Reasoner-7B               & 52.1 & 31.3 & 39.5 & 27.6 & 36.4 \\
TimeOmni-1-7B                & 54.7 & 31.0 & 36.5 & 32.2 & 36.6 \\
\midrule
\method (Ours)               & \textbf{64.3} & \textbf{43.9} & \textbf{59.3} & \textbf{33.3} & \textbf{49.8} \\
\bottomrule
\end{tabular}%
}
\end{table}

%% file: 0_sections/6_related_work.tex
\section{Related Work}

In this section, we review the key related works on the topics that are closely related to this work. We put more related works and discussions in the Appendix \ref{ap:related_work} to keep the main text concise.

\paragraph{Time-Series-only Analysis and Reasoning.}
Traditional time series analysis has predominantly addressed core tasks like forecasting, anomaly detection, and classification using specialized deep learning architectures \citep{zhou2021informer, DBLP:conf/nips/WuXWL21, DBLP:conf/iclr/NieNSK23}.
Recently, inspired by the success of LLMs, time series foundation models tokenize time series patches and are trained in a similar way to LLMs \citep{jin2024timellm, ansari2024chronos, DBLP:conf/icml/WooLKXSS24}. 
Beyond numerical sequences, some recent studies transform time series as images and leverage the reasoning abilities of Vision-Language Models (VLMs) to model time series \cite{zhong2025time,he2026harnessing}. \method focuses on a more general setting of contextualized time-series reasoning, where temporal signals are analyzed with essential contexts.

\paragraph{Contextualized and Multimodal Time Series.}
Real-world time series are heavily influenced by external events, prompting a shift to integrate external multimodal knowledge \citep{liu2025towards, DBLP:journals/corr/abs-2502-01477, sun2024test, li2024urbangpt}. 
Recent frameworks demonstrate significant improvement in forecasting accuracy with contexts \citep{DBLP:conf/ml4h/KingYM23,DBLP:conf/www/LiuHLDLHZ24}.
Building on these contextual advances, recent work has expanded from traditional numerical estimation to Time Series Question Answering (TSQA), enabling natural-language-driven description over complex temporal behaviors of the time series\citep{chen2025mtbench, wang2025chattime,jing2026tsaqa,kong2025time}. 
Yet, most existing work on contextualized time series still centers on predefined tasks such as forecasting. In contrast, \method targets generalized end-to-end workflows that require agents to jointly interpret time series, analyze with context, and produce actionable solutions accordingly.

\paragraph{Agentic Systems for Time-Series Workflows.}

As time series analysis grows in complexity, recent advancements have started applying agentic AI for time series \citep{DBLP:journals/corr/abs-2601-13653, DBLP:journals/corr/abs-2510-07432}. 
While the broader AI literature explores agents across a wide array of domains, the application of LLM agents within time series research is currently overwhelmingly concentrated on forecasting tasks and time series property analysis tasks (e.g., periodicity, trend, stationary) \citep{DBLP:journals/corr/abs-2509-11575, DBLP:journals/corr/abs-2510-01538, lee2025timecap, huang2026many}. 
While these agentic workflows show potential in forecasting accuracy and temporal property analysis, their objectives remain largely prediction-centric. \method takes a step toward a more reliable and general class of time-series agents that support context-rich and end-to-end workflow.

%% file: 0_sections/7_conclusion.tex
\section{Conclusion}
\label{sec:conclusion}

In this work, we introduce \method, a time-series-native agentic harness for reasoning over contextualized time series. \method equips generalist LLM agents with executable temporal tools, experience-driven capability evolution, and episodic multimodal memory. Across multiple time-series benchmarks, \method consistently improves over both traditional and agentic LLM baselines. These empirical results demonstrate that harnessing generalist agents with native temporal interfaces is a promising direction for building reliable and adaptable time-series intelligence.

\clearpage

%% file: 1_appendix/appendix.tex
\appendix

\onecolumn
\begin{center}
{\textbf{\LARGE Appendix}}
\end{center}
\vspace{3mm}

\noindent\textbf{Roadmap.}
In this appendix, we provide a detailed overview of our methodology and experimental setup.
Appendix~\ref{ap:analysis} motivates \method by examining why directly trusting LLMs for contextualized time-series reasoning is unreliable.
Appendix~\ref{ap:fingerprint} details the fingerprint computation underlying the multimodal memory, including the context embedding and the series fingerprint construction.
Appendix~\ref{ap:multimodal_memory} describes the full memory construction and retrieval procedure, covering bank-level normalization and the two-stage retrieval mechanism.
Appendix~\ref{ap:related_work} provides extended related work on time-series foundation models, agentic reasoning and harness design, and LLM agents for data science.
Appendix~\ref{ap:exp_details} reports experiment setup details.
Appendix~\ref{ap:full_experiments} presents the full experimental results, including component ablations across multiple backbone families, context-type and domain-wise breakdowns on CiK, memory-bank size sweeps on TSRBench, and fine-grained comparisons against open-source LLMs, VLMs, and TSLLMs.
Appendix~\ref{ap:baselines} describes the baselines and their references in detail.
Appendix~\ref{ap:case_study} provides qualitative case studies that illustrate the end-to-end behavior of \method on representative contextualized time-series tasks.
Finally, Appendix~\ref{app:prompts} provides our prompt templates.

The table of contents is provided below for quick navigation.

\renewcommand\contentsname{\Large Table of Contents}
\addtocontents{toc}{\protect\setcounter{tocdepth}{2}}


{
\setstretch{1.15}
\tableofcontents
}

\clearpage

\twocolumn

\input{1_appendix/ap_analysis}

\section{Extended Related Work}
\label{ap:related_work}

\textbf{Time-Series Foundation Models.} 
Recent time-series foundation models seek to reduce reliance on dataset-specific training through large-scale pretraining and transfer across heterogeneous temporal domains. TimeGPT demonstrates zero-shot forecasting across unseen series~\citep{garza2023timegpt}, while Lag-Llama targets probabilistic zero-shot and few-shot forecasting~\citep{rasul2023lagllama}. Other recent work scales this paradigm along complementary axes: TimesFM uses a decoder-only patched forecasting architecture~\citep{das2023decoder}, Chronos formulates forecasting as language modeling over quantized numerical tokens~\citep{ansari2024chronos}, Moirai studies universal forecasting across variables and frequencies~\citep{woo2024unified}, MOMENT broadens pretraining toward general-purpose time-series representation learning~\citep{goswami2024moment}, and Timer studies generative pretraining for forecasting~\citep{liu2024timer}. Most of the models uses transformer backbones yet other architectures are worth exploring \citep{ai2025resmoe, ai2025nirvana, zou2026transformer, tieu2025learnable}. Together, these models offer reusable pretrained components for forecasting and related time-series tasks \citep{VQKDD,guo2024}. They remain primarily model-centric, however: the interface is typically a predefined task such as forecasting, imputation, classification, or anomaly detection \citep{zhou2023one,gcformer,svq}. The interface is typically tied to predefined task formats. In contrast, \method targets open-ended contextualized time-series tasks through an agentic pipeline, where the system can interpret task goals, reason over dataset context and temporal structure, and coordinate executable analyses beyond a fixed task formulation.

\input{tables/cik_dataset}

\textbf{Agentic Reasoning and Harness.}
Recent advances in agentic systems extend LLMs from passive text generators to interactive problem solvers that operate through structured reasoning and execution harnesses~\cite{li2025survey,ning2026mc}. An agent harness defines how an LLM policy interfaces with external computation and environments, including available tools, action schemas, state representations, observation formats, and execution constraints~\cite{zhou2026externalization}. By externalizing specialized computation into this interface, harnesses reduce the burden on the LLM itself while making intermediate reasoning steps more controllable. Such harnesses are helpful to tool-augmented reasoning~\cite{ma2024sciagent} and ReAct-style agents~\cite{DBLP:conf/iclr/YaoZYDSN023}, where models decompose complex tasks into executable actions and ground intermediate reasoning in returned observations. Recent work further augments agent harnesses with memory modules for retrieving prior trajectories~\cite{hu2025memory,huang2026rethinking} and evolving mechanisms that distill reusable skills or tools from past experience~\cite{wei2025evo,fang2025comprehensive}. Based on this harness paradigm, \method specializes it for contextualized time-series reasoning through a time-series-native tool harness, episodic trajectory memory, and experience-driven tool evolution.

\textbf{LLM Agents for Data Science.}
Language models have demonstrated strong capabilities in understanding structured data \citep{li2025can, zou2025rag, ning2025graph4mm}.
LLM agents have also been developed for data analysis and machine-learning engineering, where agents write code, inspect intermediate outputs, and iterate over experiments \citep{kong2026ai, li2026heterogeneous, zhao2026papermindbenchmarkingagenticreasoning}. Data-Copilot studies code-centric querying, processing, and visualization over large data sources \citep{zhang2023datacopilot}. MLAgentBench evaluates agents that modify files, run experiments, and improve models across machine-learning tasks \citep{huang2023mlagentbench}, while MLE-bench extends this setting to Kaggle-style machine-learning engineering competitions \citep{chan2025mlebench}. Workflow-level agents such as DS-Agent and Data Interpreter further combine planning, tool use, execution feedback, and code refinement for end-to-end data-science workflows \citep{guo2024dsagent,hong-etal-2025-data}. These systems are related to \method in treating analysis as an interactive executable process rather than a fixed prediction interface. Their scope, however, is general data-science and ML-engineering automation, and time-series-specific evaluation are not the central design constraint. \method focuses on this contextualized time-series analysis setting.

\section{Experiment Setup Details}
\label{ap:exp_details}

\subsection{Dataset Statistics and Details}
\label{ap:dataset_details}

\input{tables/tsrbench_dataset}

\subsubsection{Context-is-Key (CiK) Benchmark}

We use the Context-is-Key (CiK) benchmark~\citep{DBLP:conf/icml/WilliamsAMZSRRL25} as a testbed for evaluating whether models can incorporate essential natural-language context when reasoning over contextualized time series. CiK pairs numerical temporal observations with textual information that associates with the underlying real-world process. This design aligns with our setting where time series are not treated as isolated numerical sequences, but as partial observations embedded in broader contextual evidence. The provided context may include background knowledge about the process, future events, operational constraints, historical summaries, covariates, or causal relationships.

CiK contains 71 manually designed context-aware temporal tasks. Most tasks are based on real-world time series from various application domains, including climatology, economics, energy, mechanics, public safety, transportation, and retail. The benchmark draws from 2,644 publicly available time series, with sampling frequencies ranging from 10-minute intervals to monthly observations. A small fraction of the tasks use synthetic dynamical systems. Each task template can be instantiated with different time series, time windows, and language formulations, making CiK a useful benchmark for testing whether a model can ground temporal reasoning in diverse contextual information.

Table~\ref{tab:cik_domain_stats} summarizes the task composition by domain.
CiK also annotates each task with one or more context types, including
intemporal information, historical information, covariate information, future
information, and causal information. Table~\ref{tab:cik_context_types}
summarizes these context types, their task counts, and representative examples
from the benchmark. These context-type tags are not mutually exclusive: a
single task may require multiple types of contextual information. We refer
readers to the original CiK paper for full data-source descriptions,
task-construction details, examples, and additional visualizations~\citep{DBLP:conf/icml/WilliamsAMZSRRL25}.

\begin{table*}[t]
\centering
\caption{Task composition of the TSAIA Finance benchmark used in our
experiments. The tasks cover portfolio risk, risk-adjusted return, and
market-relative stock behavior.}
\small
\resizebox{0.88\textwidth}{!}{
\begin{tabular}{l c p{0.58\textwidth}}
\toprule
Financial Task & \# Instances & Task Description \\
\midrule
VaR Confidence Level & 50 &
Compare candidate portfolios under a value-at-risk criterion and identify the
portfolio with the lowest downside risk over the target horizon. \\
\midrule
Sharpe Ratio & 50 &
Compare candidate portfolios using historical stock prices and risk-free-rate
information to determine the strongest risk-adjusted return. \\
\midrule
Market Alpha & 25 &
Assess whether a stock outperforms or underperforms a market index after
accounting for risk-adjusted market exposure. \\
\midrule
Market Beta & 25 &
Assess whether a stock is more or less sensitive to market movement by
reasoning about its market-relative volatility. \\
\midrule
Total & 150 & Applied financial temporal reasoning instances. \\
\bottomrule
\end{tabular}}
\label{tab:tsaia_finance_stats}
\end{table*}

\subsubsection{TSRBench Benchmark}

We use TSRBench~\citep{DBLP:journals/corr/abs-2601-18744} as a multi-task
benchmark for evaluating contextualized temporal reasoning with generalist
models. TSRBench is designed to assess whether models can interpret time
series together with task instructions, contextual information, and
domain-specific cues across diverse temporal reasoning scenarios. Unlike
benchmarks centered on a single predefined temporal task, TSRBench organizes
problems into multiple capability dimensions, making it well aligned with our
setting of open-ended contextualized temporal workflows.

TSRBench contains 4,125 problem instances and 15,250 time-series channels
spanning 14 domains. The benchmark covers four major reasoning dimensions:
perception, reasoning, prediction, and decision-making. These dimensions are
further divided into 15 fine-grained tasks, as summarized in
Table~\ref{tab:tsrbench_taxonomy}. TSRBench also supports multiple
representations of time series, including textual sequences, visualized plots,
text-visual inputs, and time-series embeddings, enabling evaluation of
different classes of generalist models.

Table~\ref{tab:tsrbench_taxonomy} summarizes the high-level task taxonomy and
domain coverage used in our TSRBench evaluation. The task taxonomy and domain
coverage are not treated as a strict one-to-one hierarchy: a capability
dimension may involve multiple domains, and domain coverage is summarized at a
high level based on instance-level metadata and content inspection. We refer
readers to the original TSRBench paper for detailed task construction
procedures, data collection strategies, quality control, and example
cases~\citep{DBLP:journals/corr/abs-2601-18744}.

Due to budget constraints, we evaluate on a 20\% subset of TSRBench, covering 825 out of 4,125 instances. This sampling is feasible because TSRBench contains repeated task templates; we sample balanced task coverage and verify in a small-scale study that results on the 20\% subset closely match those on the full test set.
The sampling is stratified by subtask, so each subtask contributes nearly one-fifth of its original instances to our evaluation split. Table~\ref{tab:tsrbench_split_stats} summarizes the full benchmark distribution and our evaluation subset.

\subsubsection{TSAIA Finance Benchmark}

We evaluate \method on the TSAIA Finance benchmark~\citep{DBLP:journals/corr/abs-2509-01822}, an applied financial benchmark designed around practical time-series analysis scenarios. The benchmark requires models to interpret historical financial time series together with task-specific information such as portfolio weights, risk-free rates, and market-index references. These tasks reflect realistic analytical questions in financial decision-making, including portfolio risk comparison, risk-adjusted return assessment, and market-relative stock behavior analysis.

This benchmark further evaluates the practical applicability of \method in financial time-series analysis, where temporal observations must be converted into domain-specific decisions. Table~\ref{tab:tsaia_finance_stats} summarizes the task composition used in our evaluation.

\begin{table*}[t]
\centering
\small
\caption{Tools currently exposed in the executable harness ($\mathcal{T}$) of \method, grouped by family. All tools operate on the task-local workspace $\mathcal{W}$ at full numerical precision and return compact structured results to the policy. Tools marked $\dagger$ were admitted into $\mathcal{T}$ via the capability-evolution loop of Section~\ref{sec:evolve}; the remainder are seed tools shipped with the harness.}
\label{tab:tools}
\begin{tabular}{llp{0.62\textwidth}}
\toprule
\textbf{Family} & \textbf{Tool} & \textbf{Description} \\
\midrule
\multirow{7}{*}{Inspection}
& \texttt{list\_channels}      & Return the names of the channels currently loaded in $\mathcal{W}$. \\
& \texttt{series\_overview}    & Per-channel $(n,\min,\max,\mathrm{mean})$ summary together with timestamp count and task metadata. \\
& \texttt{channel\_stats}      & Extended descriptive statistics on one channel: $n$, min, max, mean, std, median, $q_{25}$, $q_{75}$. \\
& \texttt{channel\_values}     & Bounded slice of raw values from one channel ($\le 500$ elements per call); arguments \texttt{start}, \texttt{end}, \texttt{stride}. \\
& \texttt{compute\_acf}        & Sample autocorrelation from lag $0$ to a user-specified \texttt{max\_lag}. \\
& \texttt{detect\_periodicity} & FFT-based dominant period in samples and the fraction of non-DC spectral power it carries. \\
& \texttt{find\_peaks}         & Local maxima detected via \texttt{scipy.signal.find\_peaks}; returns indices, values, and count. \\
\midrule
Forecasting
& \texttt{arima\_forecast}     & Fit ARIMA$(p,d,q)$ on a channel and return a point forecast over the next \texttt{periods} steps, together with the in-sample AIC. \\
\midrule
\multirow{3}{*}{\makecell[l]{Domain-Specific\\(Finance)}}
& \texttt{portfolio\_var}$^\dagger$       & Value-at-Risk of a weighted-price portfolio over a chosen horizon, via historical or parametric estimation. \\
& \texttt{portfolio\_sharpe}$^\dagger$    & Annualized Sharpe ratio of a weighted-price portfolio; the risk-free rate can be a constant or a per-period channel. \\
& \texttt{capm\_regression}$^\dagger$     & OLS regression of asset log returns on market log returns, returning $\alpha$, $\beta$, and $R^2$. \\
\bottomrule
\end{tabular}
\end{table*}

\subsection{Supported Time Series Tools}
\label{app:tools}

\input{tables/hyperparameter}

Table~\ref{tab:tools} catalogs the default tools exposed by the harness $\mathcal{T}$. The eight seed tools (the inspection and forecasting families) are deliberately atomic, so that the policy retains full control over how they are composed during an analysis. The remaining three, marked with $\dagger$, were \emph{not} pre-designed: they were autonomously distilled by \method's capability-evolution loop (Section~\ref{sec:evolve}) from the bank-building trajectories of prior runs.

\noindent\textbf{Evolved Tools: a case study on time series capability-evolution.} 
The three domain-specific tools, \texttt{portfolio\_var}, \texttt{portfolio\_sharpe}, and \texttt{capm\_regression}, emerged after enough successful bank-building trajectories had accumulated on TSAIA, where each task asks the agent to compare option portfolios under a specific quantitative-finance metric. The tool-synthesis agent identified three recurring sub-procedures that the policy was repeatedly assembling from primitive inspection calls: (i) empirical or parametric tail estimation of weighted-portfolio log returns (VaR); (ii) annualized mean-to-volatility ratio of risk-free-adjusted log returns (Sharpe); and (iii) OLS regression of aligned asset and market log returns (CAPM). For each pattern, the synthesis agent proposed a self-contained tool with a typed signature, a docstring, and a body that operates on the workspace through the same interface as the seed tools. Each candidate was verified by held-out execution and admitted into $\mathcal{T}$ only because its success rate exceeded $\gamma = 0.7$. The mechanism is not finance-specific: the finance trio is the first concrete instance of harness self-extension, and analogous tools (e.g., ECG-style anomaly localization, climate-baseline subtraction) are expected to emerge as $\mathcal{M}$ accumulates trajectories from other domains.

\subsection{Hyperparameters}
\label{app:hparams}

Table~\ref{tab:hparams} lists the hyperparameters of \method and the values used in our main experiments. The defaults were fixed before any aggregate result was inspected and are shared across CiK, TSRBench, and TSAIA-MC unless noted.

The subsample seed and the split seed are intentionally independent, so one can vary (e.g., for cross-validation across splits) without disturbing the other. The min-one-train-per-family policy ensures that every task family seen at test time has at least one same-family exemplar available in the bank. The bank-level standard-deviation floor catches fingerprint coordinates that are near-constant in early banks; affected coordinates collapse to zero in the $z$-scored representation rather than triggering a division-by-zero. The reference renderer omits the precedent's final answer for the anchoring reason discussed in Section~\ref{sec:memory}. All experiments use 32 concurrent rollouts; worker count affects wall-clock time but not results.

\subsection{Metrics}
\label{ap:metrics}

\paragraph{RCRPS Metric for Context-is-Key Benchmark.}
We evaluate CiK with the official Region-of-Interest Continuous Ranked
Probability Score (RCRPS) protocol~\citep{DBLP:conf/icml/WilliamsAMZSRRL25}. RCRPS is designed for context-aware temporal prediction, where the provided context may affect only a subset of the target time steps or impose constraints on valid trajectories. Compared with the standard CRPS, RCRPS introduces two context-specific
components: a region of interest (RoI), which assigns higher importance to
time steps where the context is most relevant, and a constraint-violation term,
which penalizes trajectories that violate context-implied constraints.

For a CiK task instance $\tau$, let $\widehat{\mathcal{P}}_{\tau}$ denote the
model's predictive distribution over the target temporal trajectory
$\bm{y}^{\star}_{\tau}$. We write the CiK evaluation function as
\begin{equation}
    \ell_{\mathrm{CiK}}
    (\widehat{\mathcal{P}}_{\tau}, \bm{y}^{\star}_{\tau})
    =
    \mathrm{RCRPS}_{\tau}
    (\widehat{\mathcal{P}}_{\tau}, \bm{y}^{\star}_{\tau}).
\end{equation}
Following the official CiK implementation, RCRPS combines the CRPS over the
RoI, the CRPS over the remaining target steps, and a task-specific
constraint-violation penalty:
\begin{equation}
    \mathrm{RCRPS}_{\tau}
    =
    \alpha_{\tau}\mathcal{E}_{\tau}
    +
    \mathcal{V}_{\tau}(\beta),
\end{equation}
where $\mathcal{E}_{\tau}$ denotes the RoI-weighted predictive error, $\mathcal{V}_{\tau}(\beta)$ denotes the constraint-violation penalty, and $\alpha_{\tau}$ is the task-specific scaling factor used by CiK to make scores comparable across tasks with different numerical scales. When an RoI is specified, CiK assigns equal total weight to RoI and non-RoI target steps; when no meaningful RoI is specified, the score is computed over the full target horizon. We use the official CiK evaluator without modification, including the official constraint-penalty setting $\beta=10$, task-specific RoI sets, scaling factors, constraint functions, and aggregation protocol. Lower RCRPS indicates better performance.

\vspace{3mm}
\noindent\textbf{sMAPE Metric for Context-is-Key Benchmark.}
For deterministic trajectory evaluation on CiK, let $\bm{y}^{\star}_{\tau},\widehat{\bm{y}}_{\tau}\in\mathbb{R}^{H}$ denote the gold and predicted continuations for a task instance $\tau$ over a horizon of length $H$. We use the symmetric mean absolute percentage error (sMAPE) with a small denominator floor $\varepsilon=10^{-2}$:
\begin{equation}
\mathrm{sMAPE}_{\tau}
=
\frac{1}{H}\sum_{h=1}^{H}
\frac{2|y^{\star}_{\tau,h}-\widehat{y}_{\tau,h}|}
{\max(|y^{\star}_{\tau,h}|+|\widehat{y}_{\tau,h}|,\varepsilon)}.
\end{equation}
Here $\mathrm{sMAPE}_{\tau}\in[0,2]$, and lower values indicate better
deterministic trajectory accuracy.

\subsection{Implementation Details}
\label{app:impl}

\method is implemented in Python on top of
LangChain\footnote{\url{https://www.langchain.com/}} for the LLM-agent loop and the Model Context Protocol (MCP)\footnote{\url{https://modelcontextprotocol.io/}} for the executable tool harness, with each task-local workspace $\mathcal{W}$ backed by an in-memory MCP server built using FastMCP\footnote{\url{https://gofastmcp.com/}}. Concurrent rollouts borrow isolated server--agent slots from a worker-indexed pool, so the workspace state of one task cannot leak into another.
All experiments run on a Windows 11 machine with a 13th Gen Intel(R) Core(TM) i9-13900H CPU, 64GB RAM, and an NVIDIA RTX A4500 GPU; the code is platform-agnostic and runs unchanged on Linux. No local GPU is required, as the LLM policy is queried through the OpenAI API and all in-process computation is CPU-only. The Python version is 3.11.4; dependencies are listed in the README of the supplementary code.

\section{Full Experiment Results}
\label{ap:full_experiments}

In this section, we provide additional experimental results to complement the main paper. 
Figure~\ref{fig:tsaia-ablation} extends the component ablation on TSAIA to three backbone model families. 
Across GPT, Gemini, and Claude backbones, removing tool use, capability evolution, or memory retrieval consistently degrades performance. 
This confirms that each component contributes to \method and further shows that the proposed agent harness is compatible with different model providers rather than being tied to a specific LLM family.

Figure~\ref{fig:cik-radar} provides a context-type breakdown on the CiK benchmark. 
\method achieves strong RCRPS across intemporal, historical, future, covariate, and causal information settings, indicating that the proposed framework can effectively leverage diverse forms of contextual information. 
We further report domain-wise CiK results in Table~\ref{tab:cik_resuls_domain}. 
\method obtains the best average RCRPS, ranks first in five out of eight domains, and achieves the smallest average rank, demonstrating robust performance across heterogeneous real-world time-series domains.
Figure~\ref{fig:tsrbench-bank-size} reports a memory-bank size ablation on TSRBench, sweeping the per-subtask training fraction, indicating that the multimodal memory bank is sample-efficient: a small training set already suffices to construct an effective bank, and \method is robust to the size of the training corpus used.
Finally, Table~\ref{tab:tsrbench_ap_open} provides fine-grained TSRBench results against open-source LLMs, VLMs, and TSLLMs. 
The results show that \method achieves strong performance across perception, reasoning, prediction, and decision-making subtasks, outperforming substantially larger open-source models despite using a compact backbone.

\begin{figure}[t]
    \centering
    \includegraphics[width=0.95\linewidth]{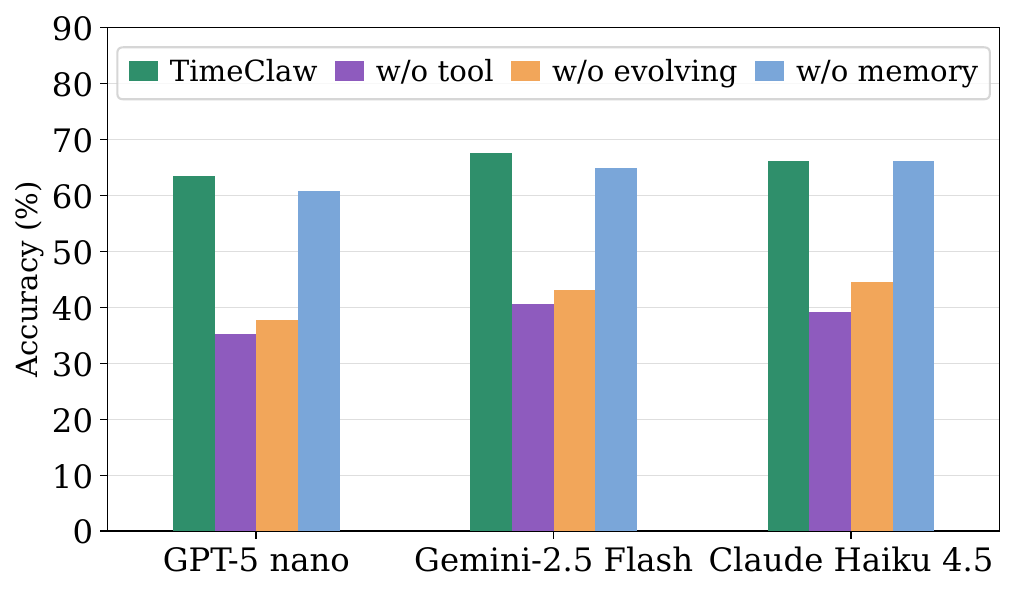}
    \caption{
Across three backbone model families, removing tool use, capability evolution, or memory retrieval consistently degrades performance, validating the contribution of each component. 
The consistent gains of \method across GPT, Gemini, and Claude backbones further show its compatibility with different model providers.
}
\label{fig:tsaia-ablation}
\end{figure}

\begin{figure}[t]
    \centering
    \includegraphics[width=0.95\linewidth]{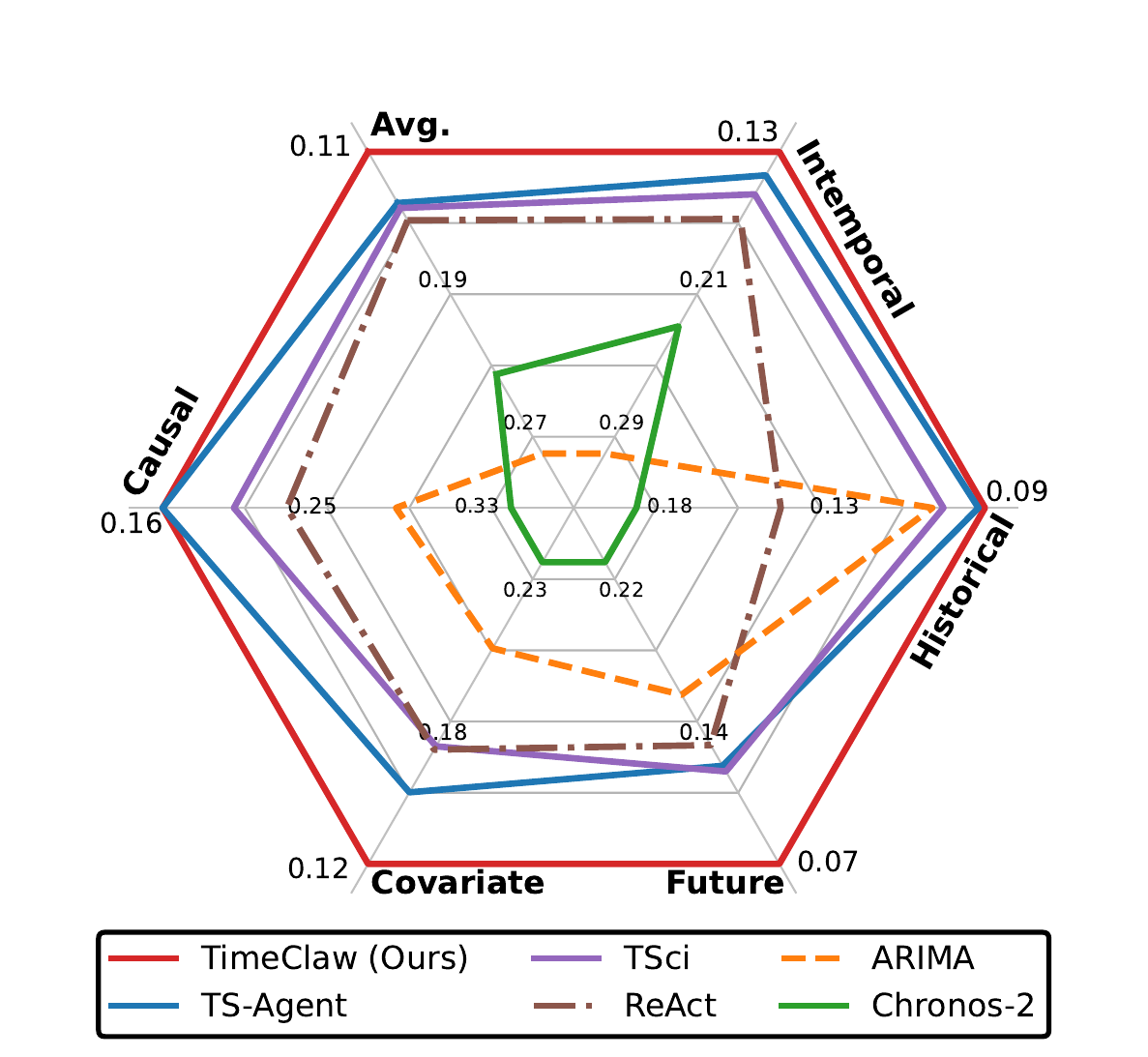}
    \caption{RCRPS comparison across CiK context types. Lower RCRPS is better; values closer to the outer boundary indicate stronger performance.}
    \label{fig:cik-radar}
\end{figure}

\input{tables/cik_domain}

\input{tables/tsr_ap_open}

\begin{figure}[t]
    \centering
    \includegraphics[width=0.95\linewidth]{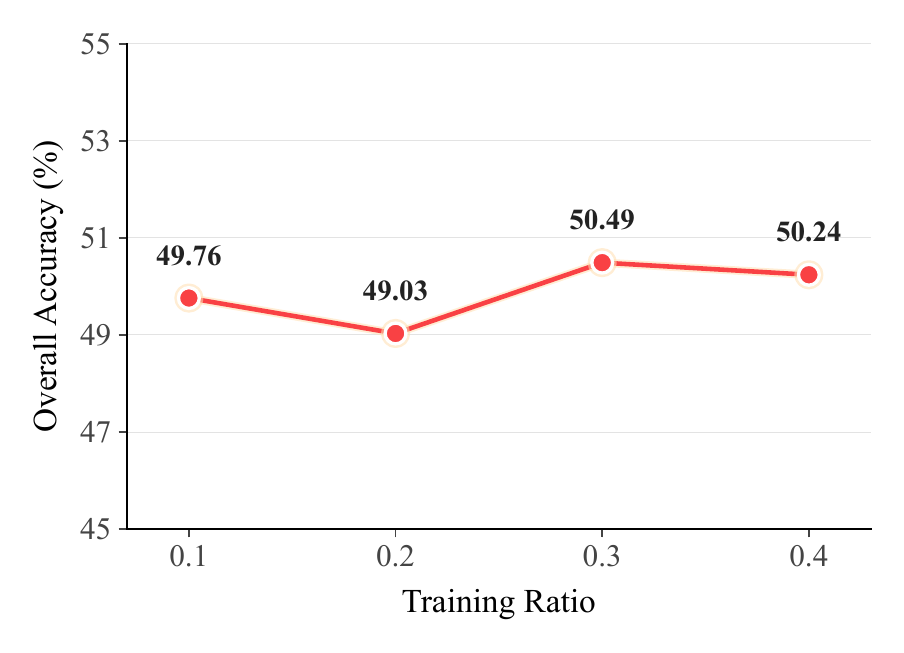}
    \caption{Overall accuracy on TSRBench as a function of the fraction of available trajectories used to construct the memory bank (with $k = 3$ retrieved neighbors at test time). Accuracy varies by less than $1.5\%$ across $\text{train\_ratio} \in \{0.1, 0.2, 0.3, 0.4\}$, indicating that the multimodal memory bank is sample-efficient: a small training set already yields a fully effective bank, and \method is robust to the construction-set size.}
    \label{fig:tsrbench-bank-size}
\end{figure}

\input{1_appendix/ap_baselines}

\input{1_appendix/case_study}


%% file: 1_appendix/ap_analysis.tex
\section{Why NOT Trusting LLMs for Contextualized Time Series?}
\label{ap:analysis}

We provide a theoretical analysis of why directly trusting LLMs over serialized
contextualized time series can be problematic. The analysis focuses on two
sources of misalignment: numerical values are first mapped into a discrete token
space, and long temporal evidence is then processed under a finite and unevenly
allocated attention budget. We begin with the tokenization-induced mismatch
between numerical distance and token distance.

Proposition~\ref{prop:token_numeric} shows that numerical closeness is not
necessarily preserved after tokenization. This provides a geometric explanation
for why surface-token similarity may be a poor proxy for numerical similarity.

\begin{proposition}[Token Distance Is Not Numerical Distance]
\label{prop:token_numeric}
Let $\eta:\mathbb{R}\rightarrow\mathcal{V}^*$ be a tokenizer-induced
representation from real values to token sequences. In general, edit distance or
embedding distance between $\eta(x)$ and $\eta(x')$ is not order-preserving with
respect to numerical distance $|x-x'|$.
\end{proposition}

\begin{proof}
Tokenizers operate over character or subword patterns rather than numerical
metric structure. Hence there can exist $x_1,x_2,x_3$ such that
\begin{equation}
    |x_1-x_2| < |x_1-x_3|,
\end{equation}
but
\begin{equation}
    d_{\mathrm{tok}}(\eta(x_1),\eta(x_2))
    >
    d_{\mathrm{tok}}(\eta(x_1),\eta(x_3)).
\end{equation}
For example, under a character-level representation,
$x_1=99.999999$, $x_2=100.0$, and $x_3=98.999999$ satisfy
$|x_1-x_2|=0.000001<1.0=|x_1-x_3|$, while $\eta(99.999999)$ and $\eta(100.0)$ differ in more character positions than $\eta(99.999999)$ and $\eta(98.999999)$ (we can alter number of 9s until this holds). Therefore, token similarity does not generally preserve numerical proximity.
\end{proof}

This mismatch is consistent with empirical findings that LLMs often struggle with
numerical reasoning tasks such as arithmetic, numerical retrieval, and magnitude
comparison, partly because numbers are represented through surface token patterns
rather than continuous magnitudes~\citep{li2025exposing}. Prior work also shows that
number-tokenization choices can substantially affect arithmetic performance,
indicating that numerical reasoning is sensitive to the discrete token
representation itself~\citep{singh2024tokenization}. Similar token-level limitations
appear in exact symbolic counting tasks: recent studies report that LLMs can fail
on simple letter-counting problems, where the model must reason over characters
rather than semantic word meaning~\citep{xu2025llm,fu2024large}.
The widely discussed ``strawberry'' counting example is an informal illustration
of the same issue: a word may be represented as subword tokens rather than as
explicit characters, so exact character-level counting is not directly aligned
with the model's token-level representation. This limitation is not unique to
``strawberry''; for example, prior work reports that an LLM can answer that the
word ``confusable'' contains two occurrences of the letter ``a'', while the
correct answer is one~\citep{xu2025llm}. Such examples suggest that even
simple character-level operations may be unreliable when they are mediated by
subword token representations rather than explicit symbolic access to
characters.

Beyond pairwise distance distortion, tokenization can also obscure the internal
structure of numbers. In decimal notation, each digit has a fixed place value,
whereas subword tokenization may group different spans of digits into variable
symbolic chunks. Proposition \ref{prop:place_value} formalizes this place-value mismatch.

\begin{proposition}[Subword Tokenization Does Not Preserve Place-Value Structure]
\label{prop:place_value}
Let $x \in \mathbb{N}$ be a decimal integer with digit expansion
\begin{equation}
    x = \sum_{k=0}^{m-1} d_k 10^k,
    \qquad d_k \in \{0,\ldots,9\}.
\end{equation}
Let $\eta(x)=(v_1,\ldots,v_J)\in\mathcal{V}^*$ be a tokenizer-induced
subword representation of the decimal string of $x$. In general, there does not
exist a fixed token-level decomposition
\begin{equation}
    x = \sum_{j=1}^{J} \phi(v_j) 10^{\alpha_j},
\end{equation}
where $\phi:\mathcal{V}\rightarrow\mathbb{Z}$ and $\alpha_j$ are tokenizer-independent
place-value exponents determined only by token position $j$.
\end{proposition}

\begin{proof}
Decimal notation has a fixed digit-level positional structure: the contribution
of digit $d_k$ is always $d_k10^k$. Subword tokenization, however, partitions the
decimal string into variable-length segments. That is, a number may be tokenized
as chunks of different lengths depending on surface string patterns, vocabulary
entries, and tokenizer merge rules.

Consider two classical numerical constants whose decimal strings may be tokenized
with different segment lengths, e.g.,
\begin{equation}
\begin{split}
    \eta(3.14159) = (3, .14, 159) \\
    \eta(2.71828) = (2.7, 18, 28)
\end{split}
\end{equation}
where tokens are written by their corresponding string pieces for illustration.
While $3.14159$ and $2.71828$ are familiar approximations of $\pi$ and $e$,
their token-level segmentations may expose different symbolic structures. In the
first case, the token ``3'' covers the integer part, ``.14'' covers the decimal
point together with the $10^{-1}$ and $10^{-2}$ places, and ``159'' covers the
$10^{-3}$, $10^{-4}$, and $10^{-5}$ places. In the second case, the token ``2.7''
covers the integer part, the decimal point, and the $10^{-1}$ place, while
``18'' and ``28'' cover the $10^{-2}$--$10^{-3}$ and $10^{-4}$--$10^{-5}$ places,
respectively. Therefore, the numerical place value associated with a token cannot
be determined solely from its token index $j$.

Hence, unless the tokenizer is explicitly constrained to respect digit-level
place values, the token sequence $\eta(x)$ does not provide a fixed token-level
basis aligned with the decimal expansion of $x$. Numerical magnitude must
therefore be reconstructed indirectly from variable-length symbolic chunks,
rather than being natively represented by the tokenization.
\end{proof}

The previous two propositions characterize representation-level mismatches: tokenization does not preserve either numerical geometry or decimal place-value structure. We next show that this mismatch propagates to the learning objective itself: predicting the next token in symbolic space is not equivalent to predicting the next value in numerical space.

\begin{proposition}[Next-Token Prediction Is Not Next-Value Prediction]
\label{prop:ntp_not_nvp}
Let $(\mathcal{V}^*, d_{\mathrm{tok}})$ be the token-sequence space induced by a
tokenizer $\eta:\mathbb{R}\rightarrow\mathcal{V}^*$, and let
$(\mathbb{R}, d_{\mathrm{num}})$ be the numerical space with
$d_{\mathrm{num}}(x,x')=|x-x'|$. Consider a temporal process
$\{X_t\}_{t\geq 1}$ with $X_t\in\mathbb{R}$ and its tokenized observation
$Z_t=\eta(X_t)$. A next-token predictor minimizes
\begin{equation}
    \mathcal{L}_{\mathrm{tok}}(p)
    =
    \mathbb{E}\left[-\log p(Z_{t+1}\mid Z_{1:t})\right],
\end{equation}
whereas a next-value predictor minimizes
\begin{equation}
    \mathcal{L}_{\mathrm{num}}(g)
    =
    \mathbb{E}\left[d_{\mathrm{num}}(g(X_{1:t}),X_{t+1})^2\right].
\end{equation}
In general, an optimizer of $\mathcal{L}_{\mathrm{tok}}$ is not an optimizer of
$\mathcal{L}_{\mathrm{num}}$.
\end{proposition}

\begin{proof}
The next-token objective is a proper scoring rule over the discrete symbolic
space $\mathcal{V}^*$; its Bayes-optimal solution is the conditional token
distribution
\begin{equation}
    p^\star(z\mid Z_{1:t})
    =
    \mathbb{P}(Z_{t+1}=z\mid Z_{1:t}).
\end{equation}
By contrast, under squared numerical loss, the Bayes-optimal next-value predictor is
\begin{equation}
    g^\star(X_{1:t})
    =
    \mathbb{E}[X_{t+1}\mid X_{1:t}].
\end{equation}
These two Bayes solutions live in different spaces: $p^\star$ is a distribution
over token strings, while $g^\star$ is an element of the numerical space
$\mathbb{R}$.

For the two objectives to be equivalent, the tokenizer-induced representation
would need to preserve the task-relevant numerical structure. In particular,
there must exist a decoding map $\delta:\mathcal{V}^*\rightarrow\mathbb{R}$ such
that token-level optimality implies numerical optimality:
\begin{equation}
    \delta\left(
    \arg\max_{z\in\mathcal{V}^*}
    p^\star(z\mid Z_{1:t})
    \right)
    =
    \mathbb{E}[X_{t+1}\mid X_{1:t}]
\end{equation}
for all conditional distributions of $X_{t+1}$. This condition does not hold in
general. First, Proposition~\ref{prop:token_numeric} shows that token distance is
not order-preserving with respect to numerical distance. Therefore, errors that
are small in token space need not be small in numerical space, and vice versa.
Second, the token objective assigns loss according to symbolic likelihood rather
than numerical deviation. For example, predicting the tokenization of $100.0$
instead of $99.999999$ may incur a large token-level discrepancy despite a small
numerical error, while predicting a token string closer in surface form can be
numerically farther away.

Hence, minimizing next-token prediction risk does not generally minimize
next-value prediction risk. The two coincide only under additional assumptions,
such as a value-sufficient tokenizer and a decoding rule that preserves the
numerical geometry relevant to the task loss.
\end{proof}

So far, the analysis concerns the numerical representation induced by
tokenization. Moreover, contextualized time series introduce another challenge: temporal
signals can be long, sparse, and distributed across many timestamps. Even when
all serialized tokens fit into the nominal context window, relevant temporal
evidence must compete with many other tokens under a fixed attention budget.

\begin{proposition}[Long-Context Temporal Information Dilution]
\label{prop:long_context_dilution}
Let a contextualized time series $\mathcal{X}=(\mathbf{X}_{1:T},\mathcal{C})$
be serialized into a token sequence
\begin{equation}
    \mathbf{z}_{1:L}=\operatorname{Tokenize}(\sigma(\mathcal{X})).
\end{equation}
Assume that a temporal task depends on a set of relevant temporal positions
$\mathcal{R}\subseteq [L]$. In a self-attention layer, the output at query
position $i$ is
\begin{equation}
    \mathbf{h}_i
    =
    \sum_{j=1}^{L}\alpha_{ij}\mathbf{v}_j,
    \quad
    \sum_{j=1}^{L}\alpha_{ij}=1,\quad \alpha_{ij}\geq 0 .
\end{equation}
If the attention mass assigned to the relevant positions is bounded by
\begin{equation}
    A_i(\mathcal{R})=\sum_{j\in\mathcal{R}}\alpha_{ij}\leq \epsilon,
\end{equation}
then the contribution of task-relevant temporal evidence to $\mathbf{h}_i$ is
at most $\epsilon$ in attention mass. Therefore, as $L$ grows, any mechanism that
spreads attention over many irrelevant tokens or concentrates attention on
position-biased sink tokens can dilute the effective use of relevant temporal
information.
\end{proposition}

\begin{proof}
Decompose the attention output into relevant and irrelevant parts:
\begin{equation}
    \mathbf{h}_i
    =
    \sum_{j\in\mathcal{R}}\alpha_{ij}\mathbf{v}_j
    +
    \sum_{j\notin\mathcal{R}}\alpha_{ij}\mathbf{v}_j .
\end{equation}
The total coefficient assigned to the relevant component is exactly
\begin{equation}
    A_i(\mathcal{R})=\sum_{j\in\mathcal{R}}\alpha_{ij}.
\end{equation}
If $A_i(\mathcal{R})\leq \epsilon$, then the relevant component contributes no
more than an $\epsilon$ fraction of the convex combination defining
$\mathbf{h}_i$. In particular, if $\|\mathbf{v}_j\|\leq M$ for all $j$, then
\begin{equation}
    \left\|
    \sum_{j\in\mathcal{R}}\alpha_{ij}\mathbf{v}_j
    \right\|
    \leq
    \sum_{j\in\mathcal{R}}\alpha_{ij}\|\mathbf{v}_j\|
    \leq
    \epsilon M.
\end{equation}
Thus, when relevant temporal positions receive little attention mass, their
effect on the attention output is bounded regardless of their task relevance.
Long serialized time series increase the number of irrelevant or weakly relevant
tokens competing for the same unit attention mass. Moreover, if some tokens act
as attention sinks and absorb disproportionate mass, the remaining mass available
to temporally relevant positions is further reduced. This proves the claim.
\end{proof}

This proposition shows that serialized temporal evidence competes within a fixed
attention budget. In long time series, relevant timestamps may be sparse, located
in the middle of the context, or separated by long lags. Empirically, LLMs often
show degraded retrieval when relevant information is placed in the middle of a
long context, a phenomenon known as ``lost in the middle'' \citep{liu2024lost, laban2025llms}.
Meanwhile, attention-sink studies show that some tokens can attract
disproportionate attention mass independent of semantic relevance \citep{xiao2024efficient, gu2025attention, kang2025see}.
Together, these effects suggest that even when a serialized time series fits
within the nominal context window, its temporally relevant information may not be
uniformly or reliably used.

\begin{corollary}[No Guaranteed Length Extrapolation]
\label{cor:length_extrapolation}
Suppose an LLM is trained on serialized sequences of length at most $L_{\max}$.
For a test-time time series whose serialization length satisfies $L>L_{\max}$,
performance on temporal dependencies requiring positions beyond the training
length is not guaranteed by the training objective alone.
\end{corollary}

\begin{proof}
The next-token training objective constrains the model on the support of the
training distribution, whose sequence lengths are at most $L_{\max}$. For
serialized sequences with $L>L_{\max}$, the model is evaluated outside this
length support. Unless the architecture, positional encoding, or training
procedure imposes an extrapolation-consistent structure, there can exist multiple
extensions of the learned conditional distribution that agree on all training
lengths but behave differently beyond $L_{\max}$. Therefore, training loss on
lengths up to $L_{\max}$ alone does not identify a unique or guaranteed behavior
for longer serialized time series.
\end{proof}

Together, these results suggest that text-only LLM workflows face both
representation-level and process-level limitations for contextualized time
series. Tokenization may distort numerical geometry and place-value structure;
next-token prediction optimizes symbolic likelihood rather than numerical
forecasting or decision risk; and long serialized contexts may dilute or miss
temporally relevant evidence. These limitations motivate time-series-native
agentic workflows that operate on temporal objects through structured analysis
rather than relying solely on serialized text.

\section{Fingerprint Computation for Multimodal Memory}
\label{ap:fingerprint}

This section provides details of the fingerprint computation of the multimodal memory in \method.

\subsection{Context Embedding}
\label{app:phi}

\paragraph{Textual descriptor $c_\tau$.}
For each task $\tau = (q, \mathcal{X}, \mathcal{Y}, y^\star, \ell)$ we form a single textual descriptor $c_\tau$ by concatenating the framing
fields available in $\mathcal{C}$ that identify \emph{what the task is
asking}, rather than the raw numerical signal. The concrete field set
depends on the benchmark schema. For Context-is-Key (CiK), $c_\tau$
joins the \texttt{background}, \texttt{scenario}, and
\texttt{constraints} fields in order. For TSRBench, $c_\tau$ joins the
subtask identifier, domain, task name, the list of channel names, and
the question text. For TSAIA, $c_\tau$ joins the
\texttt{question\_type}, the prompt, the inline data description, and
either the per-task output constraint (for analysis items) or the
option list (for multiple-choice items). Empty fields are skipped, and
the resulting string is truncated to a fixed character budget (we use
$8000$ characters, well below the encoder's $8192$-token limit) so that
pathological cases such as stringified million-element arrays cannot
overflow the encoder.

\paragraph{Encoder.}
The descriptor is embedded by a frozen sentence-level encoder,
\begin{equation}
    \phi = \Phi(c_\tau) \in \mathbb{R}^{d_\phi}.
\end{equation}
We use OpenAI's \texttt{text-embedding-3-small}, which produces
$L_2$-normalized vectors with $d_\phi = 1536$. Because the outputs are
unit-norm, cosine similarity reduces to a dot product,
$\cos(\phi_q, \phi_m) = \phi_q^\top \phi_m$, which we exploit at
retrieval time. If the encoder rejects an input as too long, the
descriptor is halved and re-submitted, with at most two rounds of
halving before the call is treated as a hard failure.

\paragraph{Storage.}
At bank-building time, the embedding is computed once per record and
stored alongside the trajectory; the bank therefore supports retrieval
without re-embedding the entire history at query time. Records logged
before the embedding mechanism was introduced carry a sentinel
no-embedding mask and are simply skipped during the text retrieval
stage.

\subsection{Series Fingerprint}
\label{app:fingerprint}

\begin{table*}[t]
\centering
\small
\caption{The 20 features of the series fingerprint $\psi$, in
assembly order. ``struct'' features describe the channel pool as a
whole; ``per-ch'' features are computed per channel and aggregated by
mean (the standard deviation is additionally log-scaled);
``mv'' is the single multivariate-coupling feature.}
\label{tab:fingerprint_features}
\begin{tabular}{rllp{0.7\textwidth}}
\toprule
\# & Symbol & Type & Description \\
\midrule
1  & $\mathrm{logLen}$              & struct & $\log_{10}(\mathrm{median}_i\, n_i + 1)$; log-scaled median per-channel finite length. \\
2  & $\mathrm{logChan}$             & struct & $\log_2(N' + 1)$; log-scaled number of non-empty channels. \\
3  & $\mathrm{miss}$                & struct & Fraction of entries in $\mathbf{X}_{1:T}$ that are non-finite (NaN or $\pm\infty$). \\
4  & $\mathrm{irreg}$               & struct & Indicator that timestamps are present and unevenly spaced ($\mathrm{CV}(\Delta t) > 0.01$). \\
5  & $\mathrm{meanZ}$               & --     & Reserved slot; identically $0$ in the current implementation. \\
6  & $\mathrm{stdLog}$              & per-ch & $\log_{10}\!\big(\tfrac{1}{N'}\sum_i \sigma_i + 1\big)$; log-scaled mean per-channel standard deviation. \\
7  & $\overline{\mathrm{IQR}/\sigma}$ & per-ch & Mean per-channel IQR-to-std ratio; a robust dispersion fingerprint. \\
8  & $\overline{\mathrm{skew}}$     & per-ch & Mean per-channel skewness on the $z$-scored signal. \\
9  & $\overline{\mathrm{kurt}}$     & per-ch & Mean per-channel excess kurtosis on the $z$-scored signal. \\
10 & $\overline{\mathrm{slope}}$    & per-ch & Mean per-channel $z$-scored linear-fit slope on a unit-rescaled time axis. \\
11 & $\overline{R^2}$               & per-ch & Mean per-channel coefficient of determination of the linear-trend fit. \\
12 & $\overline{\rho(1)}$           & per-ch & Mean autocorrelation at lag $1$. \\
13 & $\overline{\rho(\ell_2)}$      & per-ch & Mean autocorrelation at lag $\lfloor\sqrt{n_i}\rfloor$. \\
14 & $\overline{\rho(\ell_3)}$      & per-ch & Mean autocorrelation at lag $\lfloor n_i/4\rfloor$. \\
15 & $\overline{\mathrm{fftFreq}}$  & per-ch & Mean dominant non-DC frequency from the per-channel rFFT, in cycles per sample. \\
16 & $\overline{\mathrm{fftPFrac}}$ & per-ch & Mean fraction of non-DC spectral power held by the dominant bin. \\
17 & $\overline{\mathrm{specEnt}}$  & per-ch & Mean normalized Shannon entropy of the non-DC spectrum (in $[0,1]$). \\
18 & $\overline{\mathrm{cp}}$       & per-ch & Mean cumulative-sum change-point rate $Z_i / n_i$. \\
19 & $\overline{|\rho|}$            & mv     & Mean absolute Pearson correlation across distinct channel pairs ($0$ if univariate). \\
20 & $\overline{\mathrm{out}}$      & per-ch & Mean per-channel fraction of points beyond $3\,\mathrm{MAD}_i$ from the channel median. \\
\bottomrule
\end{tabular}
\end{table*}

\paragraph{Setup.}
Let $\mathbf{X}_{1:T} \in \mathbb{R}^{N \times T}$ be the numerical
signal of the task and let $\mathbf{X}^{(i)} \in \mathbb{R}^{T}$ denote
its $i$-th channel. After removing non-finite entries (NaN, $\pm\infty$)
we write the finite values of channel $i$ as
$x^{(i)}_1, \ldots, x^{(i)}_{n_i}$, where $n_i \le T$. Channels with
$n_i = 0$ are dropped, and we let $N'$ be the number of remaining
channels. A small constant $\epsilon = 10^{-12}$ is added to
denominators where indicated to keep computations finite.

A per-channel feature is set to $0$ in two degenerate cases: when
$n_i < 3$ (too short for the statistic to be meaningful) and when the
channel standard deviation $\sigma_i$ vanishes (the channel is
constant). These conventions guarantee that the aggregated fingerprint
is always finite, never NaN.

\paragraph{Per-channel statistical features.}
For each channel with $n_i \ge 3$ and $\sigma_i > 0$,
\begin{align}
\mu_i        &= \tfrac{1}{n_i} \textstyle\sum_t x^{(i)}_t, \\
\sigma_i     &= \sqrt{\tfrac{1}{n_i-1} \textstyle\sum_t (x^{(i)}_t - \mu_i)^2}, \\
\mathrm{IQR}_i &= q_{0.75}(x^{(i)}) - q_{0.25}(x^{(i)}), \\
z^{(i)}_t    &= (x^{(i)}_t - \mu_i)/\sigma_i.
\end{align}
On the $z$-scored signal we read off scale-free higher moments,
\begin{align}
\mathrm{skew}_i &= \tfrac{1}{n_i} \textstyle\sum_t (z^{(i)}_t)^3, \\
\mathrm{kurt}_i &= \tfrac{1}{n_i} \textstyle\sum_t (z^{(i)}_t)^4 - 3,
\end{align}
together with a robust outlier rate via the median absolute deviation (MAD),
\begin{align}
\widetilde{x}_i &= \mathrm{median}(x^{(i)}), \\
\mathrm{MAD}_i  &= \mathrm{median}(|x^{(i)} - \widetilde{x}_i|) + \epsilon, \\
\mathrm{out}_i  &= \tfrac{1}{n_i} \big|\{t : |x^{(i)}_t - \widetilde{x}_i| > 3\,\mathrm{MAD}_i\}\big|.
\end{align}

\paragraph{Per-channel trend features.}
A linear trend is fit on a unit-rescaled time axis so that its slope is
comparable across different-length channels. With $u_t = (t-1)/(n_i-1)$
for $t = 1, \ldots, n_i$, let
\begin{equation}
    (a_i, b_i) = \arg\min_{a, b} \textstyle\sum_t \big(x^{(i)}_t - a\, u_t - b\big)^2.
\end{equation}
We record the $z$-scored slope and the linear-fit $R^2$,
\begin{align}
\mathrm{slope}_i &= a_i / (\sigma_i + \epsilon), \\
R^2_i           &= 1 - \frac{\sum_t (x^{(i)}_t - a_i u_t - b_i)^2}
                              {\sum_t (x^{(i)}_t - \mu_i)^2 + \epsilon}.
\end{align}

\paragraph{Per-channel autocorrelation features.}
Let $c^{(i)}_t = x^{(i)}_t - \mu_i$. The sample autocorrelation at lag
$\ell \in \{1, \ldots, n_i-1\}$ is
\begin{equation}
    \rho^{(i)}(\ell)
    = \frac{\sum_{t=1}^{n_i-\ell} c^{(i)}_t c^{(i)}_{t+\ell}}
           {\sum_t (c^{(i)}_t)^2 + \epsilon}.
\end{equation}
We probe three diagnostically useful lags --- a short lag, a
square-root-scale lag, and a quarter-length lag,
\begin{equation}
    \ell^{(i)}_1 = 1, \quad
    \ell^{(i)}_2 = \lfloor \sqrt{n_i} \rfloor, \quad
    \ell^{(i)}_3 = \lfloor n_i/4 \rfloor,
\end{equation}
and record the three values $\rho^{(i)}(\ell^{(i)}_k)$ for $k=1,2,3$.

\paragraph{Per-channel spectral features.}
For $n_i \ge 8$, we compute the real FFT of the centered signal,
\begin{equation}
    \widehat{S}^{(i)}_k = \big| \mathrm{rFFT}(c^{(i)})_k \big|^2,
    \quad k = 0, 1, \ldots, \lfloor n_i/2 \rfloor,
\end{equation}
with bin frequencies $f_k = k/n_i$ in cycles per sample. Let
$k^\star_i = \arg\max_{k \ge 1} \widehat{S}^{(i)}_k$ be the dominant
non-DC bin. We record the dominant frequency and the fraction of
non-DC power it carries,
\begin{align}
\mathrm{fftFreq}_i  &= f_{k^\star_i}, \\
\mathrm{fftPFrac}_i &= \frac{\widehat{S}^{(i)}_{k^\star_i}}
                            {\sum_{k \ge 1} \widehat{S}^{(i)}_k + \epsilon},
\end{align}
together with the normalized Shannon entropy of the non-DC spectrum,
\begin{align}
p^{(i)}_k          &= \widehat{S}^{(i)}_k \Big/
                       \textstyle\sum_{k' \ge 1} \widehat{S}^{(i)}_{k'}, \\
\mathrm{specEnt}_i &= -\frac{1}{\log_2 K_i}
                       \textstyle\sum_{k \ge 1}
                       p^{(i)}_k \log_2 p^{(i)}_k,
\end{align}
where $K_i = \lfloor n_i/2 \rfloor$ is the number of non-DC bins.
Normalizing by $\log_2 K_i$ keeps $\mathrm{specEnt}_i \in [0,1]$
across series of different length. When $n_i < 8$ all three spectral
features are set to $0$.

\paragraph{Per-channel change-point feature.}
Let $d^{(i)}_t = x^{(i)}_{t+1} - x^{(i)}_t$ for $t = 1, \ldots, n_i-1$,
and define the centered cumulative-sum series
\begin{equation}
    S^{(i)}_t = \textstyle\sum_{s=1}^{t}
                 \big( d^{(i)}_s - \overline{d}^{(i)} \big),
\end{equation}
where $\overline{d}^{(i)}$ is the mean of the differences. Letting
$Z_i$ be the number of sign changes in $\mathrm{sign}(S^{(i)})$, we set
\begin{equation}
    \mathrm{cp}_i = Z_i / n_i.
\end{equation}
A stable trajectory produces a CUSUM that drifts smoothly and crosses
zero rarely, whereas a regime-shifting trajectory produces many
crossings; $\mathrm{cp}_i$ is a cheap proxy for piecewise level change.

\paragraph{Structural and multivariate features.}
The remaining features describe the series as a whole rather than any
single channel. The four structural features are
\begin{align}
\mathrm{logLen}   &= \log_{10}\!\big(\mathrm{median}_i\, n_i + 1\big), \\
\mathrm{logChan}  &= \log_2(N' + 1), \\
\mathrm{miss}     &= 1 - \frac{\sum_i n_i}{N\,T}, \\
\mathrm{irreg}    &= \mathbb{1}\!\left[
                       \frac{\sigma(\Delta\mathbf{t})}{\mu(\Delta\mathbf{t})}
                       > 0.01 \right],
\end{align}
where $\Delta\mathbf{t}$ is the sequence of inter-sample time
differences when timestamps are available, and $\mathrm{irreg} = 0$
otherwise. The single multivariate coupling feature is the mean
absolute Pearson correlation across distinct channel pairs,
\begin{equation}
    \overline{|\rho|}
    = \frac{2}{N''(N''-1)} \textstyle\sum_{i<j} |\rho_{ij}|,
\end{equation}
where $N''$ is the number of non-degenerate channels (truncated to a
common length $n_{\min} = \min_i n_i$ and excluded if zero-variance),
and $\rho_{ij}$ is the Pearson correlation between channels $i$ and
$j$. When $N'' < 2$ or $n_{\min} < 3$, $\overline{|\rho|} = 0$.

\paragraph{Aggregation across channels.}
For per-channel features other than the standard deviation, we average
across the $N'$ valid channels,
$\overline{\mathrm{feature}} = \frac{1}{N'} \sum_i \mathrm{feature}_i$.
The standard deviation is averaged and then log-scaled to prevent a
single outlier channel from dominating the metric,
\begin{equation}
    \mathrm{stdLog} = \log_{10}\!\Big( \tfrac{1}{N'}
                       \textstyle\sum_i \sigma_i + 1 \Big).
\end{equation}

\paragraph{Assembly.}
The fingerprint stacks the features in the fixed order listed in
Table~\ref{tab:fingerprint_features},
\begin{equation}
    \psi = \Psi(\mathbf{X}_{1:T}) \in \mathbb{R}^{20}.
\end{equation}
Slot 5 is reserved for a future level-aware feature and is identically
zero in the current implementation; we keep the slot so that
extensions do not invalidate existing banks.

\paragraph{Stability properties.}
Two properties of $\Psi$ are worth noting. First, $\Psi$ is invariant
under permutations of the channel index: every aggregator above is
symmetric in $i$, so a bank populated under one ordering of channels
can still retrieve a query that uses a different ordering. Second, the
use of robust statistics in the outlier-rate, IQR, and median-based
features means that replacing $O(1)$ values per channel with outliers
shifts $\psi$ by $O(1/n_i)$ on those features rather than by an
unbounded amount; the fingerprint thus changes slowly under local
corruption of the input.

\section{Full Memory Construction and Retrieval}
\label{ap:multimodal_memory}

This appendix gives the complete  memory
mechanism summarized in Section~\ref{sec:memory}: the construction of
the context embedding $\phi$, the explicit specification of the series
fingerprint $\psi$, the bank-level normalization that turns the raw
fingerprint into a retrieval-ready key, and the two-stage procedure
that combines the two modalities at query time.

\subsection{Bank-Level Normalization}
\label{app:znorm}

Because the components of $\psi$ live on very different natural scales
--- $\mathrm{logLen}$ is $O(1)$ but $\overline{\mathrm{kurt}}$ can range
over orders of magnitude --- a naive $L_2$ distance would be dominated
by whichever feature happens to have the largest variance in the bank.
We therefore $z$-score each feature using the bank's own per-feature
mean and standard deviation. Let
\begin{align}
    \boldsymbol{\mu}_\psi &= \frac{1}{|\mathcal{M}|}
        \textstyle\sum_{m \in \mathcal{M}} \psi_m, \\
    \boldsymbol{\sigma}_\psi
        &= \mathrm{std}\big(\{\psi_m\}_{m \in \mathcal{M}}\big),
\end{align}
both computed coordinate-wise. Any coordinate with
$\sigma_{\psi, d} < 10^{-9}$ is treated as constant in the bank and is
excluded from the metric by setting $\sigma_{\psi, d} \gets 1$; this
collapses that coordinate to $0$ post-normalization. The
bank-normalized fingerprint is
\begin{equation}
    \tilde\psi = (\psi - \boldsymbol{\mu}_\psi) \oslash \boldsymbol{\sigma}_\psi,
\end{equation}
where $\oslash$ denotes coordinate-wise division. The scaler is
recomputed lazily on every fresh load of the bank rather than persisted,
so that bank growth is reflected immediately at the next query.

\subsection{Two-Stage Retrieval}
\label{app:retrieval}

Given a query task $\tau_q$ with keys $(\psi_q, \phi_q)$, retrieval
proceeds in two stages that combine the two modalities.

\paragraph{Stage 1 (text-side candidate pool).}
We select an intermediate pool of size $N$ by cosine similarity in the
context-embedding space,
\begin{equation}
    \mathcal{N}(\phi_q)
    = \!\!\mathop{\mathrm{Top\text{-}}N}_{m \in \mathcal{M}^\phi}\
    \phi_q^\top \phi_m,
\end{equation}
where $\mathcal{M}^\phi \subseteq \mathcal{M}$ denotes the subset of
records that carry a valid text embedding. The dot-product form is
exact because both $\phi_q$ and the stored $\phi_m$ are unit-norm. In
our experiments we use $N = 20$.

\paragraph{Stage 2 (signal-side ranking).}
Within the text-side candidate pool, we rank by $L_2$ distance in the
bank-normalized fingerprint space,
\begin{equation}
    \mathrm{Retrieve}(\tau_q;\, k)
    = \!\!\mathop{\mathrm{Top\text{-}}k}_{m \in \mathcal{N}(\phi_q)}\
    -\big\| \tilde\psi_q - \tilde\psi_m \big\|_2.
\end{equation}
$k$ is a retrieval parameter that can be adjusted according to the LLM context window and memory density.

\paragraph{Optional filters.}
Two filters can be applied in either stage. A \emph{family filter}
restricts the candidate pool to records sharing the query's task-family
tag, which is useful for diagnostic settings that isolate the
contribution of within-family retrieval. A \emph{self-exclusion filter}
removes the query's own record, which matters when the bank and the
query set overlap (e.g.\ during cross-validation). Both filters are
off by default in our main experiments.

\paragraph{Complexity.}
Stage 1 costs $O(|\mathcal{M}| \cdot d_\phi)$ in time and is dominated
by a single dense matrix-vector product. Stage 2 costs
$O(N \cdot d_\psi)$. At our scale ($|\mathcal{M}|$ in the low
thousands, $d_\phi = 1536$, $d_\psi = 20$, $N = 20$), the entire
retrieval is sub-millisecond on a CPU and does not require an
approximate-nearest-neighbor index.

%% file: tables/cik_dataset.tex
\begin{table*}[t]
\centering
\small
\caption{Domain-level composition of the CiK benchmark. CiK contains 71
context-aware temporal tasks covering seven real-world domains and a small set
of synthetic dynamical-system tasks.}
\resizebox{1.0\textwidth}{!}{
\begin{tabular}{lcl}
\toprule
Domain / Source Type & \# Tasks & Representative Data Source \\
\midrule
Public Safety & 26 & Montreal fire department intervention logs ~\citep{ville_de_montreal_2020}\\
Climatology & 12 & Solar irradiance, cloud cover, and solar power production ~\citep{sengupta2018national} \\
Transportation & 11 & PeMS traffic occupancy data ~\citep{chen2001freeway}\\
Energy & 7 & Electricity consumption data ~\citep{godahewa2021monash}\\
Retail & 6 & ATM cash withdrawal data ~\citep{godahewa2021monash}\\
Mechanics & 3 & Causal Chambers wind-tunnel measurements ~\citep{gamella2024causal} \\
Economics & 3 & State- and county-level unemployment rates ~\citep{fred_stlouisfed}\\
Synthetic Dynamics & 3 & Simulated bivariate dynamical systems ~\citep{DBLP:conf/icml/WilliamsAMZSRRL25} \\
\bottomrule
\end{tabular}}
\label{tab:cik_domain_stats}
\end{table*}

\begin{table*}[t]
\centering
\small
\caption{Context-type composition of the CiK benchmark. Context-type tags are
not mutually exclusive, so the counts do not sum to the total number of tasks.}
\resizebox{1.0\textwidth}{!}{
\begin{tabular}{lc p{0.34\textwidth} p{0.38\textwidth}}
\toprule
Context Type & \# Tasks & Meaning in CiK & Representative Example \\
\midrule
Intemporal & 51 &
Time-invariant information about the process, such as the nature of the target variable, long-period seasonal patterns, or value constraints. &
A task may describe the temporal process and target variable, or state constraints such as positivity that are not fully inferable from the observed numerical history. \\
\midrule
Historical & 25 &
Information about past behavior that is not directly reflected in the available history window. &
A public-safety task may provide aggregate past statistics, such as the yearly average number of incidents or the month with the highest incident count. \\
\midrule
Covariate & 52 &
Information about additional variables that are statistically associated with the target series. &
In the Montreal fire examples, another incident type, such as field fires, may be provided as an associated variable that helps interpret seasonal behavior. \\
\midrule
Future & 33 &
Information about future events, scenarios, or constraints that are relevant to the future temporal behavior. &
CiK includes scenarios such as an ATM being inaccessible during a future period, leading to zero withdrawals, or an incoming weather event that changes electricity demand. \\
\midrule
Causal & 14 &
Information about whether covariates causally affect, are caused by, or are confounded with the target variable. &
In the Montreal fire examples, causal context is needed to decide whether a future intervention on one incident type should affect the target incident series. \\
\bottomrule
\end{tabular}}
\label{tab:cik_context_types}
\end{table*}

%% file: tables/tsrbench_dataset.tex
\begin{table*}[t]
\centering
\small
\caption{Task taxonomy of TSRBench. The benchmark contains 15 fine-grained
tasks grouped into four temporal reasoning dimensions. Domains are reported at
a high level based on instance-level metadata and content inspection.}
\resizebox{0.96\textwidth}{!}{
\begin{tabular}{l c p{0.25\textwidth} p{0.25\textwidth} p{0.26\textwidth}}
\toprule
Dimension & \# Tasks & Fine-Grained Tasks & Domains Covered & Main Capability Tested \\
\midrule
Perception & 4 &
Pattern Recognition; Noise Understanding; Anomaly Detection; Comparative Analysis &
Various domains; patterns drawn from synthetic signals, signal processing, and statistical time series &
Recognizing temporal patterns, noise, anomalies, and relationships among
multiple series. \\
\midrule
Reasoning & 7 &
Etiological Reasoning; Causal Discovery; Abductive Reasoning; Temporal Relation
Reasoning; Numerical Reasoning; Deductive Reasoning; Inductive Reasoning &
Health; Hydrology; Physics; Political Science; Sports; Energy; Traffic;
Agriculture; Astronomy; Chemistry; Meteorology; Seismology &
Inferring causes, latent events, temporal order, quantitative relations, rules,
and higher-level temporal principles. \\
\midrule
Prediction & 2 &
Time Series Forecasting; Event Prediction &
Finance; Climate &
Predicting future numerical values or future discrete events from temporal
observations and contextual information. \\
\midrule
Decision-Making & 2 &
Qualitative Decision-Making; Quantitative Decision-Making &
Finance; Health &
Selecting decisions or actions by combining temporal patterns, contextual
knowledge, rules, and quantitative outcome comparison. \\
\bottomrule
\end{tabular}}
\label{tab:tsrbench_taxonomy}
\end{table*}

\begin{table*}[t]
\caption{Subtask-level instance distribution in TSRBench and our evaluation
split. We use a stratified 20\% subset, selecting approximately one fifth of
the instances from each subtask.}
\centering
\small
\resizebox{0.7\textwidth}{!}{
\begin{tabular}{llcc}
\toprule
Dimension & Subtask & Full Instances & Our 20\% Split \\
\midrule
Perception & Pattern Recognition & 371 & 74 \\
Perception & Anomaly Detection & 129 & 26 \\
Perception & Similarity Analysis & 113 & 23 \\
Perception & Noise Understanding & 87 & 17 \\
\midrule
Prediction & Event Prediction & 360 & 72 \\
Prediction & Time Series Forecasting & 720 & 144 \\
\midrule
Decision-Making & Qualitative Decision-Making & 335 & 67 \\
Decision-Making & Quantitative Decision-Making & 300 & 60 \\
\midrule
Reasoning & Abductive Reasoning & 150 & 30 \\
Reasoning & Causal Reasoning & 300 & 60 \\
Reasoning & Deductive Reasoning & 250 & 50 \\
Reasoning & Etiological Reasoning & 350 & 70 \\
Reasoning & Inductive Reasoning & 100 & 20 \\
Reasoning & Numerical Reasoning & 400 & 80 \\
Reasoning & Temporal Relation Reasoning & 160 & 32 \\
\midrule
\multicolumn{2}{l}{Total} & 4,125 & 825 \\
\bottomrule
\end{tabular}}
\label{tab:tsrbench_split_stats}
\end{table*}

%% file: tables/hyperparameter.tex
\begin{table*}[t]
\centering
\small
\caption{Hyperparameters of \method.}
\label{tab:hparams}
\begin{tabular}{lll}
\toprule
\textbf{Category} & \textbf{Hyperparameter} & \textbf{Default Value} \\
\midrule
Backbone
& LLM policy $\pi_\theta$              & \texttt{gpt-5-nano} \\
\midrule
\multirow{4}{*}{Train/test split}
& Train ratio                          & $0.5$ \\
& Min train records per family         & $1$ \\
& Split seed                           & $2026$ \\
& Subsample seed                       & $42$ \\
\midrule
\multirow{5}{*}{Multimodal memory}
& Fingerprint dim $d_\psi$             & $20$ \\
& Encoder $\Phi$                       & \texttt{text-embedding-3-small} \\
& Embedding dim $d_\phi$               & $1536$ \\
& Context descriptor char cap          & $\le 8000$ \\
& Bank std floor ($z$-score)           & $10^{-9}$ \\
\midrule
\multirow{2}{*}{Two-stage retrieval}
& Text-stage pool size $N$             & $20$ \\
& Top-$k$ injected                     & $3$ \\
\midrule
Capability evolution
& Verification threshold $\gamma$      & $0.7$ \\
\bottomrule
\end{tabular}
\end{table*}

%% file: tables/cik_domain.tex
\begin{table*}[t]
\caption{Domain-wise RCRPS results on the CiK benchmark. Lower values indicate better performance. The best and second-best results in each column are highlighted in \textbf{bold} and \underline{underlined}, respectively. \method achieves the best average RCRPS with a moderate token budget across all domains, ranks first in 5 out of 8 domains, and attains the smallest average rank of 2.625.}
\label{tab:cik_resuls_domain}
\centering
\resizebox{\textwidth}{!}{%
\begin{tabular}{lccccccccc}
\toprule
 \multirow{2}{*}{\makecell{Method}} 
 & \multirow{2}{*}{\makecell{Average \\ RCRPS ($\downarrow$)}} 
 & \multicolumn{8}{c}{RCRPS by Domain ($\downarrow$)} \\
\cmidrule(lr){3-10}
 & & Climatology & Economics & Energy & Mechanics & \makecell{Public Safety} & Retail & Synthetic & Traffic \\
\midrule
\multicolumn{10}{l}{\textit{Traditional Time Series Models}} \\
~~~ARIMA \citeyearpar{broomhead1989time} & 0.2772 & 0.6717 & 0.3198 & 0.0964 & 0.7748 & 0.0938 & 0.2726 & 0.2817 & 0.2494 \\
~~~ETS \citeyearpar{jain2017study} & 0.3282 & 0.8028 & 0.2837 & 0.0971 & 0.8036 & 0.1528 & 0.3054 & 0.3011 & 0.2740 \\
~~~DLinear \citeyearpar{DBLP:conf/aaai/ZengCZ023} & 2.3448 & 1.9885 & 1.5496 & 0.1177 & 0.5057 & 5.0000 & 0.3214 & 0.3448 & 0.2425 \\
~~~PatchTST \citeyearpar{DBLP:conf/iclr/NieNSK23} & 2.3374 & 1.9903 & 1.3772 & 0.1107 & 0.5789 & 5.0000 & 0.2853 & 0.3319 & 0.2474 \\
\midrule
\multicolumn{10}{l}{\textit{Time Series Foundation Models}} \\
~~~Chronos-2 \citeyearpar{DBLP:journals/corr/abs-2510-15821} & 0.2344 & 0.1640 & 0.3439 & 0.0977 & 0.7142 & 0.1880 & 0.2990 & 0.3048 & 0.2925 \\
~~~Chronos-1-Large \citeyearpar{DBLP:journals/tmlr/AnsariSTZMSSRPK24} & 0.3362 & 0.3046 & \underline{0.2279} & 0.0610 & 0.8840 & 0.1077 & 0.2178 & 0.3450 & 1.0284 \\
~~~Chronos-1-Tiny \citeyearpar{DBLP:journals/tmlr/AnsariSTZMSSRPK24} & 0.3380 & 0.2709 & 0.2931 & 0.0609 & 0.9166 & 0.1054 & 0.2397 & 0.3538 & 1.0411 \\
~~~Chronos-1-Small \citeyearpar{DBLP:journals/tmlr/AnsariSTZMSSRPK24} & 0.3474 & 0.2796 & 0.2427 & 0.0611 & 0.8412 & 0.1155 & 0.2248 & 0.3479 & 1.1126 \\
~~~Chronos-1-Mini \citeyearpar{DBLP:journals/tmlr/AnsariSTZMSSRPK24} & 0.3640 & 0.2837 & 0.2841 & 0.0610 & 0.9219 & 0.1105 & 0.2418 & 0.3513 & 1.1834 \\
~~~Lag-Llama \citeyearpar{rasul2023lag} & 0.2628 & 0.4049 & 0.4529 & 0.0615 & 0.8911 & 0.1190 & 0.2400 & 0.3414 & 0.3435 \\
~~~Moirai-Large \citeyearpar{DBLP:conf/icml/WooLKXSS24} & 0.4545 & 1.0599 & 0.2843 & 0.0613 & 0.7277 & 0.0986 & 0.2347 & 0.3233 & 1.0133 \\
\midrule
\multicolumn{10}{l}{\textit{LLM/Agent on Time Series}} \\
~~~UniTime \citeyearpar{DBLP:conf/www/LiuHLDLHZ24} & 0.2822 & 0.6594 & 0.4928 & 0.0671 & 0.6988 & \underline{0.0931} & 0.2794 & 0.3512 & 0.2660 \\
~~~Time-LLM \citeyearpar{DBLP:conf/iclr/0005WMCZSCLLPW24} & 0.3990 & 0.8478 & 0.5408 & 0.0678 & 0.8344 & 0.1163 & 0.3276 & 0.3625 & 0.6796 \\
~~~Time-LLM w/o context \citeyearpar{DBLP:conf/iclr/0005WMCZSCLLPW24} & 0.3912 & 0.7677 & 0.5938 & 0.0672 & 0.8014 & 0.1045 & 0.3205 & 0.3622 & 0.7438 \\
~~~Llama3-70B \citeyearpar{DBLP:journals/corr/abs-2407-21783} & 0.2449 & 0.2958 & 0.2437 & 0.0611 & 0.8900 & 0.1034 & 0.1977 & 0.3583 & 0.4599 \\
\midrule
\multicolumn{10}{l}{\textit{General Agentic Pipelines}} \\
~~~TS-Agent \citeyearpar{DBLP:journals/corr/abs-2510-07432}            & 0.1421 & 0.1350 & 0.4037 & 0.0707 & 0.5449 & 0.0956 & 0.1747 & 0.1403 & 0.1070 \\
~~~TSci \citeyearpar{DBLP:journals/corr/abs-2510-01538}                & 0.1448 & \underline{0.1056} & 0.3845 & 0.0190 & 0.7685 & 0.1086 & 0.1643 & \underline{0.0878} & 0.1224 \\
~~~CoT Prompt \citeyearpar{DBLP:conf/nips/Wei0SBIXCLZ22}          & 0.1726 & 0.1282 & 0.7298 & 0.0164 & 0.9519 & 0.1248 & \textbf{0.1332} & \textbf{0.0849} & 0.1142 \\
~~~ReAct \citeyearpar{DBLP:conf/iclr/YaoZYDSN023}               & 0.1514 & 0.1078 & 0.4218 & \underline{0.0144} & 0.7077 & 0.1464 & 0.1396 & 0.1002 & 0.0929 \\
~~~Self-Reflection \citeyearpar{DBLP:journals/corr/abs-2405-06682}     & 0.1608 & 0.1129 & 0.3011 & 0.0195 & 1.1898 & 0.1255 & \underline{0.1379} & 0.0946 & 0.0983 \\
~~~Multi-agent Refine \citeyearpar{DBLP:conf/icml/YuanX25}  & \underline{0.1294} & \textbf{0.1040} & 0.4221 & 0.0320 & \underline{0.5414} & 0.1068 & 0.1456 & 0.1220 & \underline{0.0736} \\

\midrule
\method (Ours) & \textbf{0.1145} & 0.1083 & \textbf{0.1571} & \textbf{0.0103} & \textbf{0.4203} & \textbf{0.0888} & 0.1713 & 0.1304 & \textbf{0.0730} \\
\bottomrule
\end{tabular}
}
\vspace*{-0.2cm}
\end{table*}

%% file: tables/tsr_ap_open.tex
\renewcommand{\arraystretch}{1.2}
\begin{table*}[t]
\caption{
Fine-grained TSRBench accuracy (\%) of base models that are dominated by TimeClaw across all four task categories. Per-subtask numbers are taken from the original TSRBench paper~\citep{DBLP:journals/corr/abs-2601-18744}.  \textbf{Bold} marks the best score in each column; \underline{underline} marks the second-best. By harnessing GPT-5-nano, \method achieves the best performance, outperforming large open-source models with up to 235B parameters.
}
\centering
\scriptsize
\resizebox{1.0\textwidth}{!}{
\setlength{\tabcolsep}{4pt}
\begin{tabular}{l|cccc|ccccccc|cc|cc|c}
\toprule
\multicolumn{1}{l|}{\textbf{Model}}
  & \multicolumn{4}{c|}{\textbf{Perception}}
  & \multicolumn{7}{c|}{\textbf{Reasoning}}
  & \multicolumn{2}{c|}{\textbf{Prediction}}
  & \multicolumn{2}{c|}{\textbf{Decision}}
  & \multicolumn{1}{c}{\textbf{Overall}} \\
\midrule
 & PR & NU & AD & CA & ER & CD & AR & TR & NR & DR & IR & TSF & EP & QualDM & QuantDM & \\
\midrule
\multicolumn{17}{c}{\textit{Textual Time Series as Input}} \\
\midrule
Qwen2.5-3B               & 46.4 & 51.7 & 33.3 & 51.3 & 20.0 & 25.0 & 38.0 & 21.2 & 34.8 & 29.2 & 19.0 & 31.4 & 58.3 & 22.7 & 24.0 & 33.2 \\
Qwen2.5-7B               & 50.7 & 50.6 & 41.1 & 54.9 & 15.1 & 32.3 & 46.0 & 28.1 & 33.2 & 24.0 & 28.0 & 33.5 & 37.5 & 31.0 & 25.7 & 33.7 \\
Qwen2.5-72B              & 55.3 & 55.2 & 47.3 & 66.4 & 20.9 & \underline{58.0} & 62.0 & \underline{33.8} & 38.2 & 36.8 & \underline{40.0} & 30.7 & 70.3 & 34.0 & 30.3 & 42.4 \\
Qwen3-1.7B               & 45.8 & 59.8 & 38.0 & 51.3 & 18.3 & 30.1 & 48.7 & 30.0 & 27.0 & 28.4 & 20.0 & 39.2 & 67.8 & 31.3 & 24.7 & 36.8 \\
Qwen3-8B                 & 51.8 & 56.3 & 48.8 & 56.6 & 14.0 & 29.3 & 57.3 & 23.1 & 36.5 & 23.6 & 33.0 & \underline{46.7} & 48.9 & \underline{34.9} & 28.7 & 38.3 \\
Qwen3-235B-A22B          & 66.0 & 56.3 & \underline{59.7} & 67.3 & 22.6 & 34.7 & \textbf{86.7} & 28.1 & 44.8 & \textbf{49.2} & 39.0 & 29.0 & 48.9 & 34.8 & 30.0 & 42.2 \\
Gemma3-12B-it            & 51.5 & 59.8 & 48.9 & 65.5 & 21.1 & 34.7 & 42.0 & 23.7 & 33.2 & 28.0 & \textbf{41.0} & 39.4 & 34.2 & 34.3 & 31.3 & 37.2 \\
Gemma3-27B-it            & 56.1 & 59.8 & 48.1 & 68.1 & 22.3 & 37.0 & \underline{66.7} & 21.9 & 36.0 & 29.3 & 35.0 & 36.1 & 42.8 & 33.1 & 30.3 & 38.6 \\
\midrule
\multicolumn{17}{c}{\textit{Visual Time Series as Input}} \\
\midrule
Phi4-Multimodal-8B       & 52.3 & 46.0 & 41.9 & 48.7 & 24.6 & 22.3 & 28.7 & 23.1 & 30.5 & 25.2 & 26.0 & 32.1 & 25.8 & \underline{34.9} & 29.0 & 31.9 \\
InternVL3.5-1B           & 47.2 & 49.4 & 38.8 & 46.9 & 25.7 & 24.3 & 46.7 & 24.4 & 29.2 & 24.0 & 27.0 & 34.7 & 46.4 & 34.3 & 21.3 & 33.8 \\
InternVL3.5-8B           & 60.9 & 55.2 & 52.7 & 64.6 & 25.7 & 38.3 & 60.0 & 27.5 & 40.5 & 31.6 & 29.0 & 38.2 & 41.9 & 33.4 & 22.3 & 39.5 \\
InternVL3.5-38B          & 58.8 & 60.9 & 55.8 & 69.9 & 30.9 & 43.0 & 52.7 & 31.2 & 43.8 & 30.4 & 34.0 & 32.2 & 32.5 & \textbf{35.2} & 29.0 & 39.4 \\
MiniCPM-V-4.5-8B         & 63.6 & 56.3 & 56.6 & 62.8 & 22.3 & 24.3 & 54.7 & 20.6 & 26.5 & 22.8 & 24.0 & 44.6 & 27.8 & 26.3 & 29.7 & 35.9 \\
MiMo-VL-7B-RL            & 58.2 & \underline{64.4} & 52.7 & 65.5 & 23.7 & 34.7 & 65.3 & 30.6 & 33.0 & 25.6 & 36.0 & 35.0 & 29.7 & 29.0 & 33.0 & 37.2 \\
\midrule
\multicolumn{17}{c}{\textit{Embedded Time Series as Input}} \\
\midrule
OpenTSLM-3B      & 39.9 & 40.2 & 35.7 & 43.4 & 20.3 & 32.3 & 33.3 & 28.1 & 24.5 & 26.0 & 32.0 & 32.5 & 35.6 & 28.4 & 28.7 & 31.0 \\
ChatTS-14B               & 50.7 & 50.6 & 46.5 & 55.8 & 21.7 & 34.3 & 51.3 & 24.4 & 30.5 & 23.6 & 25.0 & 19.2 & 52.2 & 29.3 & 27.7 & 33.5 \\
TS-Reasoner-7B           & 53.1 & 56.3 & 48.1 & 50.4 & 28.9 & 24.3 & 57.3 & 26.9 & 37.2 & 23.6 & 24.0 & 31.7 & 55.0 & 32.8 & 21.7 & 36.4 \\
TimeOmni-1-7B            & 55.0 & 59.8 & 41.9 & 64.6 & 28.0 & 35.3 & 46.7 & 24.4 & 31.5 & 22.8 & 34.0 & 30.0 & 49.4 & 34.0 & 30.3 & 36.7 \\
\midrule
TimeClaw (GPT-5-nano)    & \underline{76.9} & 22.2 & 45.5 & \underline{72.7} & \underline{45.7} & 40.0 & 33.3 & 25.0 & \underline{72.5} & 32.0 & 10.0 & \textbf{47.2} & \underline{83.3} & 21.2 & \textbf{46.7} & \underline{49.8} \\
TimeClaw (GPT-5-mini)    & \textbf{84.6} & \textbf{66.7} & \textbf{81.8} & \textbf{81.8} & \textbf{48.6} & \textbf{76.7} & 53.3 & \textbf{50.0} & \textbf{80.0} & \underline{44.0} & \underline{40.0} & 38.9 & \textbf{86.1} & 21.2 & \underline{33.3} & \textbf{57.3} \\
\bottomrule
\end{tabular}
}
\label{tab:tsrbench_ap_open}
\end{table*}
\renewcommand{\arraystretch}{1.0}

%% file: 1_appendix/ap_baselines.tex
\section{Baseline Details and References}
\label{ap:baselines}

Due to the broad range of baselines considered in our experiments, we provide additional details in this appendix, including how each baseline is adapted to our evaluation settings when necessary. For complete methodological descriptions, we refer readers to the corresponding original publications.

\paragraph{Traditional Time Series Models.}
This group includes classical statistical models and neural models that operate directly on numerical time-series histories. These methods serve as non-agentic baselines for evaluating how far standard temporal modeling can go without explicit contextual reasoning or agentic workflow support.

\begin{itemize}
   \item \textbf{ARIMA}.
    ARIMA is a classical statistical model that captures temporal dependence through autoregressive, differencing, and moving-average components.

    \item \textbf{ETS}.
    ETS uses exponential smoothing to model level, trend, and seasonal
    components. Like ARIMA, it relies on historical numerical patterns rather than natural-language context or agentic reasoning.

    \item \textbf{DLinear}~\citep{DBLP:conf/aaai/ZengCZ023}.
    DLinear decomposes a time series into trend and remainder components using a moving average, applies separate one-layer linear mappings to them, and sums the outputs for prediction.

    \item \textbf{PatchTST}~\citep{DBLP:conf/iclr/NieNSK23}.
    PatchTST is a Transformer-based model for long-term time-series modeling. It segments each time series into subseries-level patches and treats these patches as input tokens to the Transformer. It also uses a channel-independent design, where different variables share the same embedding and Transformer weights.
\end{itemize}

\paragraph{Time Series Foundation Models.}
This group contains pretrained foundation models designed for broad
time-series modeling across datasets and domains. These baselines test whether specialized temporal pretraining alone can match a harnessed agent that uses temporal tools, contextual reasoning, reusable routines, and memory.

\begin{itemize}
    \item \textbf{Chronos-family models}~\citep{DBLP:journals/tmlr/AnsariSTZMSSRPK24, DBLP:journals/corr/abs-2510-15821}.
    Chronos formulates time-series forecasting as language modeling over discrete value tokens: continuous observations are scaled, quantized into a fixed vocabulary, and modeled with Transformer language-model architectures. Through pretraining on diverse real and synthetic time series, Chronos-family models provide strong zero-shot probabilistic forecasting baselines.

    \item \textbf{Lag-Llama}~\citep{rasul2023lag}.
    Lag-Llama is a pretrained foundation model for univariate probabilistic time-series forecasting. It uses a decoder-only Transformer architecture with lagged values as covariates and is pretrained on diverse time-series datasets to support zero-shot and fine-tuned prediction. 
    \item \textbf{Moirai-family models}~\citep{DBLP:conf/icml/WooLKXSS24}.
    Moirai is a universal time-series forecasting Transformer trained for cross-domain zero-shot prediction. It is designed to handle heterogeneous temporal data, including different sampling frequencies, variable counts, and distributional properties, through a masked encoder architecture.
\end{itemize}

\paragraph{Language Models.}
This group includes general-purpose language models and time-series-oriented language-model baselines. These methods help evaluate whether language modeling or language-based adaptation alone is sufficient for contextualized temporal reasoning, without the full time-series-native harness used by \method. In our experiments, we compared with a broader set of such general-purpose language models show improved performance over them. We provide details of representative models below. We refer the readers to the original benchmark papers \cite{DBLP:conf/icml/WilliamsAMZSRRL25, DBLP:journals/corr/abs-2601-18744} for more details.

\begin{itemize}
    \item \textbf{LLaMA3-70B}~\citep{DBLP:journals/corr/abs-2407-21783}.
    LLaMA3-70B is a strong open-weight general-purpose LLM from the Llama 3 family. For CiK, we use it as a text-only language-model baseline: the numerical time series and natural-language context are serialized into a single prompt, and the model generates the continuation autoregressively through next-token prediction. Under greedy decoding, each generated token is selected according to the model's output logits.

    \item \textbf{UniTime}~\citep{DBLP:conf/www/LiuHLDLHZ24}.
    UniTime is a language-empowered unified model for cross-domain time-series forecasting. It uses domain instructions to provide domain-level information and a Language-TS Transformer to align language and time-series representations. The model aims to handle heterogeneous time series from different domains under
    a unified architecture.

    \item \textbf{Time-LLM}~\citep{DBLP:conf/iclr/0005WMCZSCLLPW24}.
    Time-LLM reprograms pretrained LLMs for time-series forecasting while keeping the backbone language model frozen. It maps time-series patches into text-prototype representations aligned with the LLM input space, augments the input with Prompt-as-Prefix task information, and projects the LLM outputs into temporal predictions.
\end{itemize}

\paragraph{Agentic Frameworks.}
This group includes general prompting and agentic workflows that are not specifically designed for time series. They test whether generic LLM reasoning, tool-use, or reflection pipelines are sufficient for contextualized temporal tasks without specialized temporal runtime support.

\begin{itemize}
    \item \textbf{Direct prompting.}
    Direct prompting feeds the serialized time series, context, and task instruction to the language model in a single prompt. The model directly produces the final answer without explicit reasoning steps, tool use, reflection, or multi-agent interaction.

    \item \textbf{Chain-of-thought prompting}~\citep{DBLP:conf/nips/Wei0SBIXCLZ22}.
    Chain-of-thought prompting asks the language model to produce intermediate natural-language reasoning steps before the final answer. It is a prompting-only strategy for eliciting multi-step reasoning from large language models.

    \item \textbf{ReAct}~\citep{DBLP:conf/iclr/YaoZYDSN023}. 
    ReAct interleaves natural-language reasoning traces with task-specific actions. The reasoning steps help the model maintain and update its plan, while actions allow it to interact with external tools or environments. It is a general agentic prompting baseline for combining reasoning and tool use.

    \item \textbf{Self-reflection}~\citep{DBLP:journals/corr/abs-2405-06682}.
    Self-reflection lets an LLM inspect its previous answer or reasoning process, identify possible errors, and use the generated feedback to revise its response. The authors study several reflection variants where agents learn guidance from prior mistakes before attempting the task again.

    \item \textbf{Multi-agent reflection}~\citep{DBLP:conf/icml/YuanX25}.
    Multi-agent reflection uses multiple LLM agents to iteratively verify and
improve answers. It models the refinement process as a Markov Decision Process and trains an actor-critic LLM system to incorporate feedback over multiple rounds, improving reasoning through coordinated agent collaboration.
\end{itemize}

\paragraph{Time Series Agents.}
This group includes agentic methods specifically designed for time-series
analysis. These baselines are closest to our setting, but they typically focus
on particular temporal tasks or fixed agentic workflows rather than providing a
general harness for contextualized temporal reasoning.

\begin{itemize}
    \item \textbf{TS-Agent}~\citep{DBLP:journals/corr/abs-2510-07432}.
    TS-Agent is a time-series reasoning agent that lets LLMs synthesize conclusions from evidence while delegating statistical and structural extraction to tools over raw numerical sequences. It maintains an explicit evidence log and uses a self-critic with a final quality gate to iteratively refine reasoning.

    \item \textbf{TSci}~\citep{DBLP:journals/corr/abs-2510-01538}.
    TSci is an LLM-driven agentic framework for general time-series forecasting. It organizes the workflow into specialized agents for data diagnostics, preprocessing, model planning, fitting, validation, ensembling, and reporting. The framework emphasizes automated forecasting together with transparent natural-language rationales and reports.
\end{itemize}

%% file: 1_appendix/case_study.tex
\section{Case Study}
\label{ap:case_study}

\addtocontents{toc}{\protect\setcounter{tocdepth}{1}}

\begin{figure*}[t]
\centering
\includegraphics[width=\linewidth]{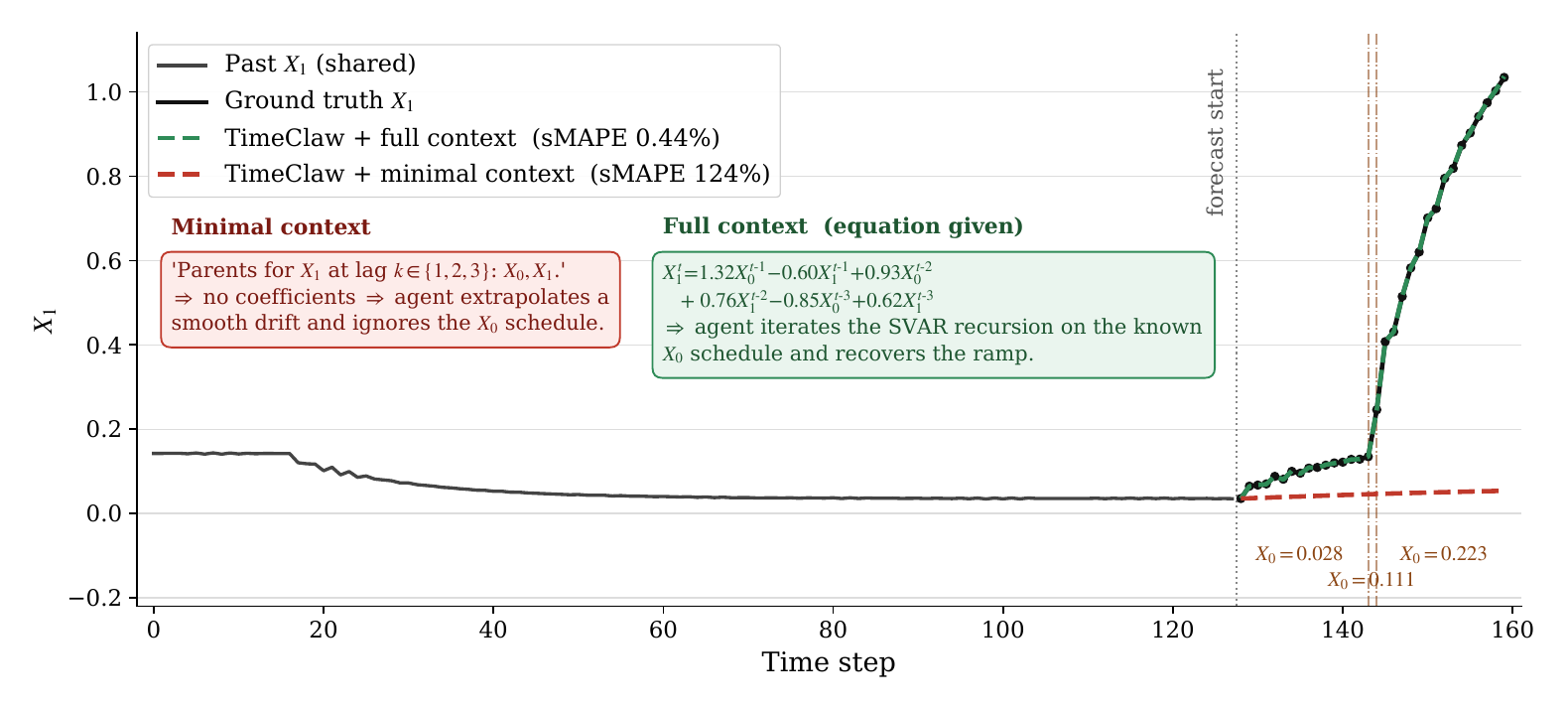}
\caption{Case study on a paired SVAR task from Context-is-Key \citep{DBLP:conf/icml/WilliamsAMZSRRL25}. Given identical history and ground truth, \method{} recovers the ramp when the prompt contains the SVAR equation and misses it entirely when only the qualitative parent list is provided. This demonstrates that \method{} actively grounds its forecast in the textual context, rather than extrapolating from the numerical history alone.}
\label{fig:cik_case}
\vspace{-3mm}
\end{figure*}

\subsection{How \method Leverages Context?}

To test how \method{} leverages textual context rather than
pattern-matching on numerical history, we pair two SVAR tasks at the same
seed: \textsc{MinimalCausalContextBivarLinSVAR} and
\textsc{FullCausalContextExplicitEquationBivarLinSVAR}. The two records
have byte-identical 128-step past history of $X_1$, identical 32-step
ground truth, and the same step-change schedule for the covariate $X_0$
($0.028\!\to\!0.111\!\to\!0.223$). The Minimal prompt only lists the
qualitative parents of $X_1$; the Full prompt additionally spells out the
linear SVAR equation with all six coefficients. Any difference in
forecast quality must therefore route through the agent's reading of
roughly $400$ extra tokens of context.

Figure~\ref{fig:cik_case_study_svar} overlays both forecasts on the
shared ground truth. Given the equation, the agent inspects recent $X_1$
values through the analysis tools and unrolls the SVAR recursion on the
known future $X_0$ schedule, recovering the ramp from $0.035$ to $1.03$
within $\sim\!10^{-3}$ (sMAPE $\mathbf{0.44\%}$). Without the
coefficients it cannot translate the $X_0$ jumps into $X_1$ movements
and falls back to a smooth drift inside $[0.036, 0.054]$, missing the
ramp entirely (sMAPE $\mathbf{124\%}$). The $\sim$280 times gap on
identical numerical inputs confirms that, when context carries
decision-relevant structure, \method{} propagates that structure into
the forecast rather than relying on the history alone.

\begin{figure*}[h]
\centering
\begin{subfigure}[t]{0.32\textwidth}
  \includegraphics[width=\linewidth]{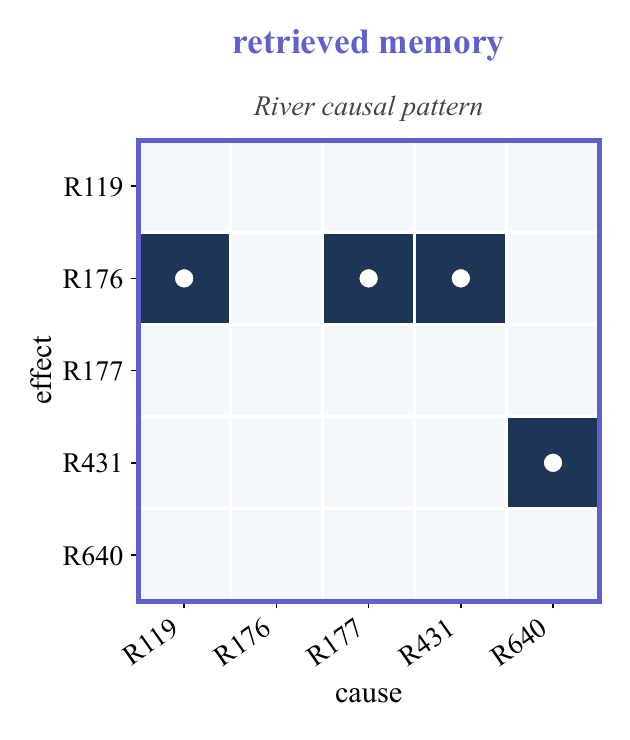}
\end{subfigure}\hfill
\begin{subfigure}[t]{0.32\textwidth}
  \includegraphics[width=\linewidth]{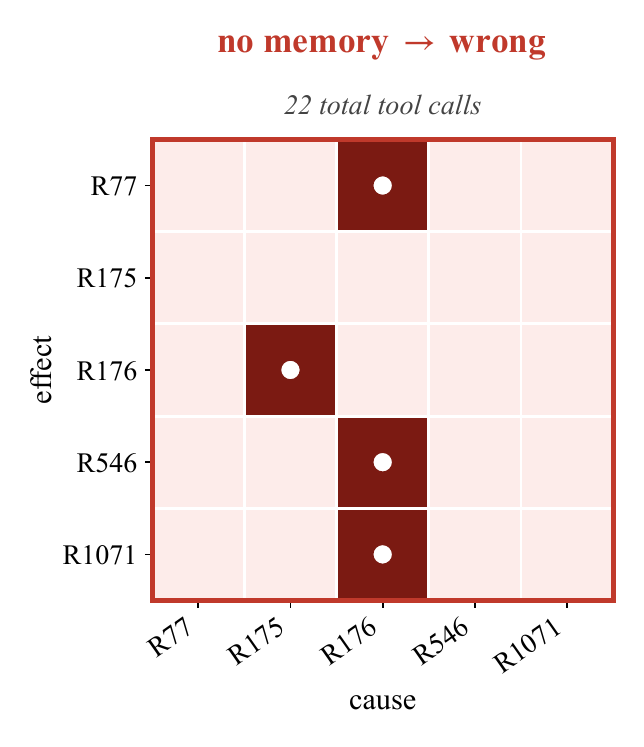}
\end{subfigure}\hfill
\begin{subfigure}[t]{0.32\textwidth}
  \includegraphics[width=\linewidth]{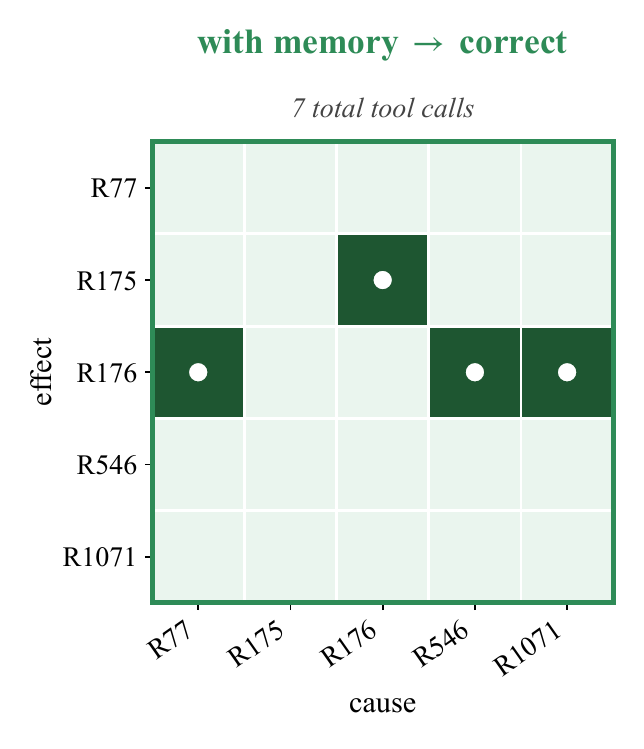}
\end{subfigure}
\caption{
Causal maps of the case study on memory transfer in TSRBench river flood causal reasoning (R stands for "River"). 
\textbf{Left:} a retrieved memory summarizes an upstream-to-downstream cascade from a related but different river system, providing a transferable structural rule. 
\textbf{Middle:} without memory, the agent performs broad exploratory analysis and commits to a physically reversed adjacency after $22$ tool calls. 
\textbf{Right:} with memory, the agent reuses the retrieved upstream-to-downstream reasoning pattern, executes only $7$ targeted tool calls, and selects the correct causal structure. 
This illustrates that episodic memory improves reasoning by transferring compact task-level rules across related temporal scenarios, rather than by matching numerical histories alone.
}
\label{fig:tsrbench_case_study_causal}
\end{figure*}

\subsection{How Task Memory Steers Agent Decisions}

To examine how retrieved task memory changes agent behavior, we trace
\method{} on a representative TSRBench causal-reasoning instance. The task asks
the agent to infer the causal adjacency among multiple river sensors from the
same underlying time series. We run the task under two settings with the same
agent, MCP tool budget, and numerical input: retrieval disabled ($k\!=\!0$) and
top-$3$ memory retrieval enabled ($k\!=\!3$). The retrieved memories come from
the training bank and do not overlap with the test split.

Figure~\ref{fig:tsrbench_case_study_causal} contrasts the two outcomes. 
Without memory, the agent performs a broad numerical sweep, issuing $22$ tool
calls across \texttt{channel\_stats}, \texttt{compute\_acf},
\texttt{find\_peaks}, and \texttt{detect\_periodicity} for all channels.
Despite this exhaustive inspection, it selects a physically reversed causal
structure, mistaking downstream effects for upstream causes. 

With memory, retrieval returns related causal-reasoning trajectories from
sibling tasks. One retrieved memory summarizes a transferable rule: the relevant
mechanism follows an upstream-to-downstream cascade. Guided by this rule, the
agent avoids an exhaustive scan, runs only $7$ targeted \texttt{compute\_acf}
probes to verify the lead-lag pattern, and selects the correct adjacency. 
Importantly, the retrieved memory provides no additional observations from the
test instance; it transfers a compact decision rule learned from a different
river system. Thus, memory changes both the answer and the reasoning process,
reducing the tool budget by over $3\times$ while steering the agent toward the
correct causal direction. A step-by-step trace is provided in
Figure~\ref{fig:tsrbench_case_study_causal_box}.

\subsection{How \method{} Evolves New Tools from Experience?}

We further examine how \method{} expands its toolset through experience-driven capability evolution. 
Tool evolution converts recurring successful trajectories into reusable executable routines. 
This mechanism is especially useful when the initial toolbox only provides generic time-series operations, while a benchmark repeatedly requires domain-specific computations that are cumbersome or unreliable to reconstruct from low-level tools. 
In such cases, \method{} summarizes successful training trajectories, identifies repeated analytical routines, and promotes them into callable tools with standardized argument schemas. 
The evolved tools can then be reused at test time, allowing the agent to bypass brittle manual arithmetic and directly execute the intended computation.

Figure~\ref{fig:tsaia_case_study_finance_tools} illustrates this process on TSAIA's finance split. 
Starting from a generic time-series toolbox, the agent observes recurring successful trajectories involving portfolio risk-adjusted return estimation, portfolio risk estimation, and market-factor regression. 
From these trajectories, \method{} evolves three finance-specific tools: \texttt{portfolio\_sharpe}, \texttt{portfolio\_var}, and \texttt{capm\_regression}. 
At test time, these evolved tools provide direct executable support for financial reasoning, improving both accuracy and tool-use efficiency. 
This case study shows that capability evolution transfers reusable computational procedures, rather than task-specific data, and enables the agent to acquire new domain-specific skills from prior experience.

\section{Prompt Templates}
\label{app:prompts}

\method instantiates a small family of prompts that share the same high-level
skeleton but specialize the input rendering, the output schema, and the
tool-use guidance to the underlying benchmark. All test-time prompts can
optionally be wrapped with a retrieval prefix that injects top-$k$ trajectory
summaries from the episodic memory bank, and all training-time prompts are
extended with a ground-truth-aware suffix that elicits a transferable
\texttt{<context\_to\_action>} explanation block. We describe the unified
skeleton in Figure~\ref{fig:prompt_template_general}, and an example per-benchmark
instantiations in Figure~\ref{fig:prompt_tsrbench}. We further describe the memory bank prompt in Figure ~\ref{fig:prompt_augments}.

\clearpage

\begin{figure*}[h]
\centering
\resizebox{\textwidth}{!}{
\begin{tcolorbox}[colback=gray!5!white, colframe=pink!120,
title=Case Study C.1: \method{} on Context-is-Key: How Textual Context Steers the Forecast,
boxrule=0.3mm, width=1.2\textwidth, arc=2mm, auto outer arc=true]

\textbf{Core Task.}
We pick a pair of CiK tasks from the bivariate linear SVAR family at the
same seed (\texttt{seed=5}): \textsc{MinimalCausalContextBivarLinSVAR} and
\textsc{FullCausalContextExplicitEquationBivarLinSVAR}. The two records
share \emph{identical} 128-step past history and 32-step ground-truth
future of the target variable $X_1$, and share the same step-change
schedule for the covariate $X_0$ over the forecast window
($0.028 \!\to\! 0.111 \!\to\! 0.223$). The \emph{only} difference is what
the textual context reveals about the data-generating process. This
isolates a clean variable: context $\to$ forecast.

\vspace{2mm}

\textbf{Prompt:} See Appendix~\ref{app:prompts} for the full prompt template.

\vspace{2mm}

\textbf{Visualization:} We provide Figure \ref{fig:cik_case} to better illustrate this case study.

\vspace{3mm}
\hrule
\vspace{3mm}

\textbf{Memory Retrieval Stage.}
\begin{enumerate}
    \item \textbf{Two-stage retrieval.}
    \method{} embeds the task's
    \texttt{background + scenario + constraints} block with
    \texttt{text-embedding-3-small}, takes the top-$K_{\text{text}}$ records
    by cosine similarity from the training bank, then re-ranks the
    survivors by L2 distance on the 20-dim numerical fingerprint of
    \texttt{past\_time}. With $k\!=\!1$ on this task both stages converge
    on the same family (\textsc{FullCausalContextImplicitEquationBivarLinSVAR})
    seen during training.

    \item \textbf{Reference compression.}
    The retrieved trajectory is compressed by \texttt{summarize\_trajectory} into the trainer's analytic spine: a \texttt{context$\to$forecast} explanation, together with the sequence of MCP tool calls and their truncated responses. 
    This summary is then injected before the test prompt as a \texttt{REFERENCES FROM PRIOR TRAINING} block, surfacing the transferable rule that the scenario's lagged linear parent effects should be handled by iterating the SVAR recursion over the known future $X_0$ schedule.
\end{enumerate}

\textbf{Execution Stage (Full Context).}
\begin{enumerate}
    \item \textbf{Context parsing.}
    The agent reads the explicit linear SVAR equation from the prompt
    ($X_1^{t} = 1.322\,X_0^{t\text{-}1} - 0.604\,X_1^{t\text{-}1} +
    0.926\,X_0^{t\text{-}2} + 0.763\,X_1^{t\text{-}2} -
    0.851\,X_0^{t\text{-}3} + 0.623\,X_1^{t\text{-}3}$) and the
    piece-wise constant schedule for $X_0$ over the 32-step horizon.

    \item \textbf{Tool-grounded historical inspection.}
    Guided by the retrieved spine, the agent invokes the MCP analysis
    server: \texttt{series\_overview()} $\to$ \texttt{channel\_stats("1")}
    $\to$ \texttt{compute\_acf("1", max\_lag=20)} $\to$ \texttt{channel\_values("1", 120, 128)},
    confirming the recent $X_1$ levels needed to seed the recursion.

    \item \textbf{Closed-form forecast.}
    With the equation, the recent lagged values, and the future $X_0$
    schedule all known, the agent unrolls the recursion deterministically,
    producing a 32-step forecast that tracks the ground-truth ramp from
    $0.035 \!\to\! 1.03$ almost exactly.
\end{enumerate}

\textbf{Counterfactual: Minimal Context, Same Series.}
On the matched \textsc{Minimal}-context variant, the prompt only states that
``parents for $X_1$ at lag $k\!\in\!\{1,2,3\}$ are $X_0, X_1$,'' while the
six SVAR coefficients are withheld. 
Given the same retrieved memory and tool budget, the agent cannot reconstruct the
recursion and instead falls back to extrapolating a small monotone drift in
$[0.036, 0.054]$, missing the $X_0$ step changes that drive the true ramp.

\vspace{2mm}
\textbf{Evaluation.}
On the full-context variant, \method{} achieves \textbf{sMAPE 0.44\%}
(MSE $\approx\!10^{-6}$), compared with \textbf{sMAPE 124\%}
(MSE $0.254$) on the minimal-context variant. 
This yields a $\sim\!280\times$ sMAPE gap for forecasting the
\emph{same} numerical series. 
Since the two variants share identical histories, retrieved memory, and tool
budget, the gap is explained by the $\sim\!400$-token difference in textual
context. 
This demonstrates that \method{}'s gains arise from context comprehension rather
than distributional pattern matching on the history alone. 
Figure~\ref{fig:cik_case_study_svar} overlays both forecasts on the shared past
and ground truth.

\end{tcolorbox}}
\caption{
Case study of \method{} on a pair of CiK SVAR tasks that share their
numerical history and ground truth but differ only in the textual
description of the data-generating process. With the explicit SVAR
equation in context, the agent iterates the recursion on the known $X_0$
step schedule and recovers the ground-truth ramp (sMAPE $0.44\%$). With
only a qualitative parent list, it extrapolates a smooth low-amplitude
drift and misses the ramp entirely (sMAPE $124\%$). The two annotated
boxes inside the panel reproduce the exact context delta between the two
prompts.
}
\label{fig:cik_case_study_svar}
\end{figure*}

\clearpage

\begin{figure*}[h]
\centering
\resizebox{\textwidth}{!}{
\begin{tcolorbox}[colback=gray!5!white, colframe=brown!55,
title=Case Study C.2: \method{} on TSRBench: How Retrieved Memory Steer the Decision,
boxrule=0.3mm, width=1.2\textwidth, arc=2mm, auto outer arc=true]

\textbf{Core Task.}
A multiple-choice causal-discovery question from TSRBench (\texttt{causal\_reasoning}, Elbe-River-Flood). The agent is shown $64$ days of runoff measurements from five river sensors and must pick the $5\!\times\!5$ binary adjacency matrix that captures the causal structure (columns = cause, rows = effect). Ground-truth answer is D: an upstream-to-downstream cascade in which small tributaries (rivers R77, R546, R1071) feed the large downstream river R176, which in turn feeds R175.

\vspace{2mm}

\textbf{Prompt:} See Appendix~\ref{app:prompts} for the full template.

\vspace{2mm}
\textbf{Visualization:} We provide Figure \ref{fig:tsrbench_case_study_causal} to better illustrate this case study.

\vspace{3mm}
\hrule
\vspace{3mm}

\textbf{Memory Retrieval Stage.}
\begin{enumerate}
    \item \textbf{Two-stage retrieval.}
    \method{} embeds the test prompt with \texttt{text-embedding-3-small} and selects the top-$K_{\text{text}}$ records from the training bank by cosine, then re-ranks the survivors by $L2$ distance on the fingerprint of the loaded series. Then \method retrieves $k\!=\!3$ records.

    \item \textbf{Rule extraction via \texttt{summarize\_trajectory}.} Each retrieved trajectory is compressed to a tool-call spine plus a \texttt{context\_to\_action} block written during training. One block of the retrieved action in memory reads verbatim: \emph{``the mechanism mapping is upstream-to-downstream edges.''} This is the transferable decision rule that primes the test agent.

    \item \textbf{Prompt injection.}
    The compressed references are concatenated as a \texttt{REFERENCES FROM PRIOR TRAINING} block in front of the test prompt, only analytic spine and decision rules.
\end{enumerate}

\textbf{Execution Stage.}
\begin{enumerate}
    \item \textbf{Targeted tool use.}
    Primed by the retrieved rule, the agent runs only the inspections needed to confirm direction: \texttt{list\_channels} $\to$ \texttt{series\_overview} $\to$ five \texttt{compute\_acf} probes, one per river. It skips the exhaustive per-channel statistics, peak-finding, and periodicity-detection that the no-memory agent runs through.

    \item \textbf{Mapping the rule onto the new river set.} The agent matches its own observations to the retrieved rule: the channels with mean flow $\sim\!660$--$680$ (R175, R176) form the downstream sinks; the three with mean flow $<\!4$ (R77, R546, R1071) are upstream tributaries. The rule then determines option D, in which the sink row R176 is densely populated with parents from the tributaries and the tributaries themselves are source rows.

    \item \textbf{Answer realisation.} The agent emits the final letter \texttt{<answer>D</answer>}, with a brief justification citing the retrieved cascade rule.
\end{enumerate}

\textbf{Counterfactual: Same Series, No Memory.}
Re-running the identical task at $k\!=\!0$ removes only the \texttt{REFERENCES} block. Without memory, the agent resorts to brute-force inspection, issuing $22$ tool calls across \texttt{channel\_stats}, \texttt{compute\_acf}, \texttt{find\_peaks}, and \texttt{detect\_periodicity} for every channel. It then selects option A, which encodes the physically reversed cascade (R176 caused by R175; R77, R546, and R1071 each caused by R176). Without the retrieved rule, the agent lacks an anchor for causal direction and defaults to a topology inconsistent with the observed lead-lag patterns.

\vspace{2mm}
\textbf{Evaluation.} On the same instance, \method{} changes from \textbf{wrong (A)} at $k\!=\!0$ (without memory) to \textbf{correct (D)} at $k\!=\!3$ (with memory), while reducing the tool-call budget from $22$ to $7$ ($3.1\!\times$ fewer calls). 
The retrieved memory provides a transferable \emph{decision rule}. 
This demonstrates memory improves reasoning by carrying compact rules that redirect the agent's search process, rather than by matching numerical histories alone.

\end{tcolorbox}}
\caption{Case study on a TSRBench river-flood causal-discovery instance}
\label{fig:tsrbench_case_study_causal_box}
\end{figure*}

\begin{figure*}[h]
\centering
\resizebox{\textwidth}{!}{
\begin{tcolorbox}[colback=gray!5!white, colframe=violet!45,
title=Case Study C.3: \method{} on TSAIA: Finance Tools Evolved from Experienced Trajectories,
boxrule=0.3mm, width=1.2\textwidth, arc=2mm, auto outer arc=true]

\textbf{Core Task.}
TSAIA's finance split presents the agent with portfolio and market time-series panels and asks it to answer quantitative multiple-choice questions. At the beginning, the agent's toolbox only contains generic time-series inspection tools. It does not include finance-specific primitives for portfolio evaluation or market-factor regression. 

\vspace{2mm}

\textbf{Prompt:} See Appendix~\ref{app:prompts} for the full template.

\vspace{3mm}
\hrule
\vspace{3mm}

\textbf{Tool Evolution.}
\begin{enumerate}
    \item \textbf{Routine extraction via \texttt{summarize\_trajectory}.}
    During training, three recurring analytical routines appear across successful trajectories: portfolio risk-adjusted return estimation, portfolio risk estimation, and market-factor regression. 
    Each successful trajectory is summarized into a compact routine description that records the input conventions, intermediate computations, and final decision rule. 
    These summaries reveal reusable computation patterns.

    \item \textbf{Evolving finance-specific tools.}
    From these recurring routines, \method{} evolves three finance-specific
    tool schemas and adds them to the agent's reusable toolbox:
    \vspace{-1mm}
    \begin{itemize}\setlength\itemsep{0pt}
      \item \texttt{portfolio\_sharpe(channels, weights, risk\_free, period\_per\_year)}: compute risk-adjusted return for each
        portfolio option and compare the returned ratios.
      \item \texttt{portfolio\_var(channels, weights, horizon=$H$, alpha=$\alpha$, method="parametric")}: compute portfolio risk using task metadata and compare the resulting losses across options.
      \item \texttt{capm\_regression(asset\_channel, market\_channel)}: estimate market-factor statistics $\{\alpha,\beta,r^{2}\}$ and select the option whose target statistic matches the returned value.
    \end{itemize}
\end{enumerate}
\textbf{Evaluation.}
Across the finance test split, evolved tools improve both accuracy and tool-use efficiency, especially when manual arithmetic is unreliable. On {MarketAB-beta} tasks, the average number of \texttt{capm\_regression} calls per task increases to $0.83$, and accuracy rises by $+16.7$. On {VaR} tasks, accuracy improves by $+12.0$, while the average tool budget drops from $6.3$ to $5.2$.

\end{tcolorbox}}
\caption{
Case study of \method{} on TSAIA's finance MC split. 
Starting from a generic time-series toolbox, \method{} observes recurring
successful training trajectories and evolves finance-specific tools for
portfolio evaluation and market-factor regression. 
At test time, these evolved tools replace brittle manual arithmetic with
direct executable routines, improving accuracy while reducing unnecessary
exploratory calls. 
}
\label{fig:tsaia_case_study_finance_tools}
\end{figure*}

\clearpage

\begin{figure*}[h]
\centering
\resizebox{\textwidth}{!}{
\begin{tcolorbox}[
colback=gray!5!white, colframe=blue!70,
title=Prompt Template B.1: General Skeleton,
boxrule=0.3mm, width=1.2\textwidth, arc=2mm, auto outer arc=true]

\textbf{General Prompt.} All \method prompts decompose the input into a fixed
sequence of optional slots: a retrieval prefix, a tool-use hint (when the
agent is wired to the MCP analysis server), a benchmark-specific header, the
task question, an inline data block (only when the data is not preloaded
into the MCP server), an option list (for multiple-choice tasks), an
output-format clause, and a training-mode suffix. Concrete builders fill or
omit each slot according to the task type.

\begin{quote}
\texttt{\{retrieval\_prefix\}}

\texttt{\{tools\_hint\}}

\texttt{\{header\}}

\medskip
\texttt{\{question\_or\_context\}}

\medskip
\texttt{\{inline\_data\_block\}}

\medskip
\texttt{\{options\}}

\medskip
\texttt{\{output\_format\}}

\medskip
\texttt{\{training\_suffix\}}
\end{quote}

\textbf{Template fields.}
\texttt{\{retrieval\_prefix\}} is the optional ``REFERENCES FROM PRIOR
TRAINING'' block (Fig.~\ref{fig:prompt_augments}, left), inserted at test
time when the memory bank returns nearest-neighbor trajectories.
\texttt{\{tools\_hint\}} reminds the agent which MCP analysis tools are
available and the recommended inspection pattern.
\texttt{\{header\}} carries benchmark-level metadata.
\texttt{\{question\_or\_context\}} contains the task statement (TSRBench /
TSAIA) or the textual context fields (CiK).
\texttt{\{inline\_data\_block\}} renders the partial numerical series as text when needed, and is
often suppressed, since the canonical data is preloaded
into the MCP server.
\texttt{\{options\}} is the lettered option list only for multiple-choice tasks.
\texttt{\{output\_format\}} pins the response schema.
\texttt{\{training\_suffix\}} is the optional extension
(Fig.~\ref{fig:prompt_augments}, right) that converts the prompt into a
trajectory-collection prompt during memory-bank construction.
\end{tcolorbox}}
\caption{
General prompt skeleton used in \method. A small set of optional slots
covers all benchmarks; each concrete builder fills the slots appropriate
to its task type and tool-availability regime.
}
\label{fig:prompt_template_general}
\end{figure*}

\begin{figure*}[h]
\centering
\resizebox{\textwidth}{!}{
\begin{tcolorbox}[
colback=gray!5!white, colframe=blue!70,
title=Prompt Template B.2: TSRBench,
boxrule=0.3mm, width=1.2\textwidth, arc=2mm, auto outer arc=true]

\textbf{TSRBench Prompt.} TSRBench questions come in three surface forms:
(i)~``perception'' questions with an inline \texttt{<ts><ts/>} placeholder
where the series should be substituted, (ii)~multiple-choice questions with
a \texttt{choices} field, and (iii)~free-form questions. The same template
handles all three by toggling the rendering of the series block and the
options block; the answer schema is always a single
\texttt{<answer>...</answer>} tag.

\begin{quote}
\texttt{Domain: \{domain\}} \\
\texttt{Task type: \{task\}} \\
\texttt{Series names: \{names\}}

\medskip
The numerical time series for this question is already loaded into your
analysis tool server. Inspect it with the available tools before answering.
A typical pattern: (1) \texttt{list\_channels()} / \texttt{series\_overview()};
(2) \texttt{channel\_stats} / \texttt{channel\_values} / \texttt{compute\_acf} /
\texttt{detect\_periodicity} / \texttt{find\_peaks} on individual channels.
Do NOT guess values you haven't retrieved through tools.

\medskip
\texttt{Question:} \\
\texttt{\{question\}}

\medskip
\textit{(MC subtasks only)}\\
\texttt{Options:} \\
\texttt{A. \{option\_A\}} \\
\texttt{B. \{option\_B\}} \\
\texttt{...}

\medskip
Reason carefully, then output your final answer wrapped exactly like this on
its own line: \texttt{<answer>X</answer>}. 
\end{quote}
\end{tcolorbox}}
\caption{
Specific TSRBench prompt template. For Context-is-Key and TSAIA prompt template, please refer to our code implementation. 
}
\label{fig:prompt_tsrbench}
\end{figure*}

\begin{figure*}[h]
\centering
\resizebox{\textwidth}{!}{
\begin{tcolorbox}[
colback=gray!5!white, colframe=blue!70,
title=Prompt Template B.3: Memory-Retrieval Prefix and Training-Mode Suffix,
boxrule=0.3mm, width=1.2\textwidth, arc=2mm, auto outer arc=true]

\textbf{Augmentation blocks.} Two cross-cutting blocks compose with every
builder above. The \textit{retrieval prefix} is inserted at test time when
the memory bank returns nearest-neighbor trajectories; the \textit{training
suffix} is appended at memory-bank-construction time so the agent produces
a justifying analytic trace plus a transferable
\texttt{<context\_to\_action>} explanation.

\begin{quote}
\textbf{Retrieval prefix (test time, when $k>0$ neighbors are retrieved):}

\texttt{=== REFERENCES FROM PRIOR TRAINING ===} \\
Below are similar prior tasks (same task family where possible) and how they
were solved, including which tools were called, the key numbers returned, and the
final answer. Use them as a guide for the analytic process. The correct
answer for the current task may differ; do not copy blindly.

\medskip
\texttt{Reference 1:} \\
\texttt{[family=<family\_key>]} \\
\texttt{\ \ <trainer's transferable explanation>} \\
\texttt{\ \ tool\_name(args) $\to$ <truncated\_response>} \\
\texttt{\ \ tool\_name(args) $\to$ <truncated\_response>} \\
\texttt{\ \ ...}

\medskip
\texttt{Reference 2: ...}

\medskip
\texttt{=== END REFERENCES ===}

\medskip\hrule\medskip

\textbf{Training-mode suffix (memory-bank construction):}

\texttt{TRAINING MODE}

\medskip
Do TWO things before emitting your final answer:

1. Inspect the series with whatever tools are available and walk through the
reasoning step by step.

2. Output a \texttt{<context\_to\_action>...</context\_to\_action>} block
($\leq$3 sentences) that explicitly states:
\begin{itemize}
\item which sentences in the context (background / scenario / constraints /
question / options) drive the answer's shape, AND
\item the rule that maps those sentences to that shape (e.g.\ ``scenario says
`heat wave for 2 hours' $\to$ ground truth has a 4$\times$ spike at the
stated start $\to$ because air-conditioning load scales with cooling
demand'').
\end{itemize}

Make this block transferable, NOT series-specific. Future test-time
agents will read it to map their own context, so avoid statements like ``the mean is 850'' and prefer statements like ``scenario X implies pattern Y because of mechanism Z''.

\medskip
Then produce your final answer in the format requested above.

\end{quote}
\end{tcolorbox}}
\caption{
The retrieval prefix (top) surfaces the
analytic spine of nearest-neighbor trajectories from the memory bank.
The training-mode suffix (bottom) is applied during the memory bank
construction, so the trainer agent produces the transferable explanation
block that the retrieval prefix later surfaces.
}
\label{fig:prompt_augments}
\end{figure*}